\documentclass[runningheads]{llncs}

\usepackage[utf8]{inputenc}
\usepackage[T1]{fontenc}
\usepackage{url}
\usepackage{booktabs}
\usepackage{amsfonts}
\usepackage{nicefrac}
\usepackage{microtype}
\usepackage[table]{xcolor}
\definecolor{definitionorange}{RGB}{220,128,54}
\definecolor{definitionorangelight}{RGB}{255,245,235}
\usepackage{graphicx}

\usepackage[numbers,sort&compress]{natbib}

\usepackage[T1]{fontenc}
\usepackage[scale=0.9]{sourcecodepro}

\usepackage{amssymb, amsmath}
\usepackage{xspace}
\usepackage{bm}
\usepackage[caption=false,font=footnotesize]{subfig}

\usepackage{array}
\usepackage{makecell}
\newcolumntype{L}[1]{>{\raggedright\arraybackslash}p{#1}}
\newcolumntype{C}[1]{>{\centering\arraybackslash}p{#1}}

\usepackage{graphicx}
\usepackage{booktabs}
\usepackage[ruled,vlined]{algorithm2e}

\usepackage{paralist}
\usepackage{adjustbox}
\usepackage{float}

\usepackage{proof}

\usepackage{xparse}

\usepackage[most]{tcolorbox}
\tcbuselibrary{theorems}

\newtcbtheorem[number within=section]{boxeddefinition}{Definition}{
  enhanced,
  breakable,
  colback=white,
  colframe=black!60,
  boxrule=0.6pt,
  arc=1.2mm,
  left=6pt,right=6pt,top=6pt,bottom=6pt,
  fonttitle=\bfseries,
  coltitle=black,
  attach title to upper,
  boxed title style={
    colback=black!3,
    colframe=black!60,
    boxrule=0.6pt,
    arc=1.2mm
  }
}{def}

\usepackage[inline]{enumitem}

\usepackage{stmaryrd}

\usepackage{tikz}
\usetikzlibrary{shapes.geometric}
\usetikzlibrary{arrows.meta,arrows,positioning,calc}
\usetikzlibrary{automata}

\usepackage{hyperref}

\newcommand{\Ag}{\mathsf{Ag}}

\newcommand{\Ap}{\mathsf{Ap}}

\newcommand{\true}{\ensuremath{\top}}
\newcommand{\false}{\ensuremath{\bot}}

\newcommand{\LTL}{\ensuremath{\mathsf{LTL}}}
\newcommand{\SafeLTL}{\ensuremath{\LTL_{\mathsf{safe}}}}

\newcommand{\Next}{\ensuremath{\mathsf{X}}}

\newcommand{\WeakUntil}{\ensuremath{\mathsf{W}}}
\newcommand{\Release}{\ensuremath{\mathsf{R}}}

\newcommand{\Always}{\ensuremath{\mathsf{G}}}

\newcommand{\defeq}{\mathrel{\overset{\mathrm{def}}{=}}}

\DeclareMathOperator{\Supp}{Supp}
\newcommand{\WinningRegion}{\ensuremath{\mathsf{Win}^{\Box}}}

\renewcommand{\iff}{\ensuremath{\leftrightarrow}}

\newcommand{\IPath}{\ensuremath{\mathsf{IPath}\xspace}}
\newcommand{\Trace}{\ensuremath{\mathcal{T}}}

\definecolor{shieldgrey}{RGB}{92,98,105}
\definecolor{shieldgreylight}{RGB}{247,248,249}

\newcommand{\shieldbadge}{%
  \tikz[baseline=-0.50ex, x=0.58ex, y=0.58ex]{%
    \path[fill=shieldgrey]
      (0,1.18) -- (0.92,0.78) -- (0.72,-0.10)
      .. controls (0.56,-0.74) and (0.18,-1.10) .. (0,-1.28)
      .. controls (-0.18,-1.10) and (-0.56,-0.74) .. (-0.72,-0.10)
      -- (-0.92,0.78) -- cycle;
  }%
}
\newcommand{\ShieldIcon}{\mathord{\raisebox{-0.08ex}{\mbox{\shieldbadge}}}}

\newcommand{\ContractShieldMask}[1]{\mathcal{M}_{\ShieldIcon}^{\mathcal{C},#1}}

\renewcommand{\iff}{\ensuremath{\leftrightarrow}}

\newcommand\env[1]{\textsc{#1}}
\newcommand\algo[1]{\texttt{#1}}

\newcommand{\AG}[3]{\ensuremath{\langle #1 \rangle #2 \langle #3 \rangle}\xspace}

\usepackage{thmtools,thm-restate}
\newlength{\logicsemanticsrhswidth}
\newlength{\logicsemanticslhswidth}
\newlength{\logicsemanticstextinset}
\newcommand{\logicsemanticssetrhswidth}[1]{%
    \setlength{\logicsemanticsrhswidth}{#1}%
}

\newcolumntype{Y}[1]{>{\raggedright\arraybackslash}p{#1}}

\newcommand{\logicsemanticssetboxstyle}[1]{%
    \tcbset{logicsemantics/.style={#1}}%
}

\logicsemanticssetboxstyle{%
    enhanced,
    breakable,
    colframe=black!75!white,
    colback=black!4!white,
    arc=0mm,
    boxrule=0.95pt,
    left=0mm,
    right=0mm,
    top=0.8mm,
    bottom=0.8mm,
    boxsep=0pt
}

\newcommand{\logicsemanticsintrofont}{%
    \footnotesize\sffamily\bfseries
}

\newcommand{\logicsemanticssectionfont}{%
    \footnotesize\sffamily\bfseries
}

\newcommand{\logicsemanticsbodyfont}{%
    \scriptsize
}

\newcommand{\logicsemanticsfullrow}[1]{%
    \multicolumn{3}{@{}Y{\dimexpr\linewidth-2\logicsemanticstextinset\relax}@{}}{%
        \hspace*{\logicsemanticstextinset}\rule{0pt}{2.45ex}#1%
    }\\[-0.15ex]%
}

\providecommand{\logicsemanticsintro}{%
    For any valuation $\nu$, history $h$, path $\pi$, position $i \geq 1$, and formulae of the logic under study:%
}

\newenvironment{logicsemantics}[1][\logicsemanticsintro]{%
    \par\medskip
    \noindent
    \begin{tcolorbox}[logicsemantics]
    \logicsemanticsbodyfont
    \setlength{\tabcolsep}{0pt}%
    \renewcommand{\arraystretch}{0.96}%
    \begin{tabular*}{\linewidth}{@{}Y{\logicsemanticslhswidth} c Y{\logicsemanticsrhswidth}@{}}%
    \logicsemanticsfullrow{\logicsemanticsintrofont #1}%
    \noalign{\vskip 0.55ex}%
}{%
    \end{tabular*}%
    \end{tcolorbox}%
    \par\medskip
}

\newcommand{\logicsemanticsdivider}{%
    \noalign{\hrule height 0.4pt}%
}

\newcommand{\logicsemanticssection}[1]{%
    \logicsemanticsdivider
    \logicsemanticsfullrow{\logicsemanticssectionfont #1}%
    \logicsemanticsdivider
    \noalign{\vskip 0.7ex}%
}

\newcommand{\logicsemanticsrule}[3]{%
    #1 & #2 & #3\\%
}

\newcommand{\logicsemanticsruleblock}[3]{%
    #1 & #2 & \begin{tabular}[t]{@{}l@{}}#3\end{tabular}\\%
}

\makeatletter
\DeclareRobustCommand{\RestatedTitleSuffix}{%
  \ifthmt@thisistheone\else\space(restated)\fi
}
\makeatother

\begin{document}

\title{Contract-Based Compositional Shielding for Safe Multi-Agent Reinforcement Learning}
\titlerunning{Contract-Based Compositional Shielding}
\author{Omar Adalat\inst{1} \and
Edwin Hamel-De le Court\inst{1,2}\and
 Francesco Belardinelli \inst{1}}
\authorrunning{O. Adalat et al.}
\institute{Imperial College London, London SW7 2AZ, UK \and
University of Manchester, Manchester M13 9PL, UK
\email{\{o.adalat24,e.hamel.de-le-court,francesco.belardinelli\}@imperial.ac.uk} 
\\}
\maketitle

\begingroup
\renewcommand{\thefootnote}{}
\footnotetext{Paper accepted to the 23rd European Conference on Multi-Agent Systems.}
\endgroup

\begin{abstract}
\emph{Safe coordination} problems surface in multi-agent reinforcement learning when global safety cannot be enforced by any agent unilaterally: the admissibility of one agent's action may depend on the dynamics of other agents. Decentralised shields can enforce safety at runtime, but purely factorised permissions often exclude optimal team behaviour that is safe only through coordination. We study deterministic safety guarantees for agents trained and deployed under decentralised execution, recovering team-optimal safe behaviour without centralised runtime control. Agents have a shared global specification $\phi$ in the safety fragment of Linear Temporal Logic ($\SafeLTL$), and select among tuples of local $\SafeLTL$ obligations whose conjunction implies the global specification $\phi$. Each agent may rely on the other agents' local obligations as assumptions because the whole contract tuple is certified simultaneously and allows projection into local action masks. At learning time, a non-stationary multi-armed bandit chooses among a library of local $\SafeLTL$ obligations to select the tuple that optimises team reward, all without forgoing end-to-end safety. We evaluate the approach across 6 environments and 15 algorithmic variants.
\end{abstract}

\keywords{Safe Multi-Agent Reinforcement Learning  \and Decentralised Learning \and Compositional Verification \and Shielding}

\section{Introduction} \label{sec:intro}
Assuring the safety of learning cooperative agents requires reconciling two demands: agents should optimise a shared task objective, whilst satisfying safety constraints during training and deployment. In \emph{safe coordination} problems~\cite{elsayed2021safe,raja2009towards}, the admissibility of one agent's action depends on the non-stationary policies of the other agents in the shared environment, so reasoning that treats teammate choices as arbitrary must discard behaviour that is safe only under coordinated actions. \emph{Multi-Agent Reinforcement Learning} (MARL) offers a powerful framework for sequential decision-making under uncertainty in a shared environment, by leveraging sampling-based methods to iteratively refine policies in stochastic games~\cite{littman1994markov}. Stochastic games are commonly studied under cooperative, competitive, and mixed strategic dynamics. Many multi-agent systems naturally involve cooperative tasks, such as rescue drones~\cite{drew2021multi} and autonomous warehouses~\cite{wurman2008coordinating}, which makes the cooperative setting a natural target for safe MARL. However, standard reward penalties are insufficient to verify that behaviour is safe once a policy is deployed~\cite{ji2023safety}. From the formal methods standpoint, \emph{shielding}~\cite{konighofer2017shield} is a popular technique for enforcing safety during both training and deployment by pre-emptively masking out unsafe actions that may lead to a specification violation or post-posedly replacing unsafe actions~\cite{alshiekh2018safe}.

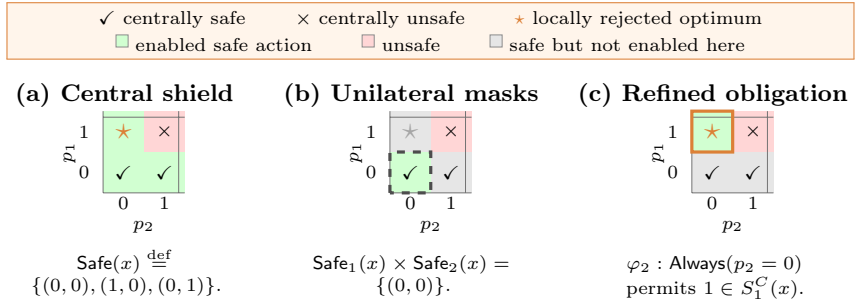
\begin{figure}[t]
  \centering
  \begingroup
\centering
\begin{center}
    \setlength{\fboxsep}{3pt}
    \fcolorbox{definitionorange}{definitionorangelight}{%
    \begin{minipage}{0.9\textwidth}
    \centering
    \scriptsize
    \checkmark\ centrally safe
    \qquad $\times$ centrally unsafe
    \qquad \textcolor{definitionorange}{$\star$} locally rejected optimum

    \vspace{0.25em}
    \tikz[baseline=-0.45ex]\fill[green!18, draw=black!45, line width=0.25pt] (0,0) rectangle (0.16,0.16);
    enabled safe action
    \qquad
    \tikz[baseline=-0.45ex]\fill[red!16, draw=black!45, line width=0.25pt] (0,0) rectangle (0.16,0.16);
    unsafe
    \qquad
    \tikz[baseline=-0.45ex]\fill[black!10, draw=black!45, line width=0.25pt] (0,0) rectangle (0.16,0.16);
    safe but not enabled here
    \end{minipage}}
\end{center}

\vspace{-0.35em}
\begin{adjustbox}{max width=\linewidth}
\begin{tabular}{@{}c@{\hspace{1.6em}}c@{\hspace{1.6em}}c@{}}
\textbf{(a) Central shield} &
\textbf{(b) Unilateral masks} &
\textbf{(c) Refined obligation} \\
\begin{tikzpicture}[
    x=0.54cm,
    y=0.54cm,
    cell/.style={draw=black!55, line width=0.3pt},
    label/.style={font=\scriptsize}
]
    \fill[green!18] (0,0) rectangle (1,1);
    \fill[green!18] (1,0) rectangle (2,1);
    \fill[green!18] (0,1) rectangle (1,2);
    \fill[red!16] (1,1) rectangle (2,2);
    \draw[cell] (0,0) grid (2,2);
    \node[label] at (0.5,-0.26) {$0$};
    \node[label] at (1.5,-0.26) {$1$};
    \node[label] at (1,-0.82) {$p_2$};
    \node[label, anchor=east] at (-0.12,0.5) {$0$};
    \node[label, anchor=east] at (-0.12,1.5) {$1$};
    \node[label, rotate=90] at (-0.82,1) {$p_1$};
    \node[font=\large, text=definitionorange] at (0.5,1.5) {$\star$};
    \node[label] at (0.5,0.5) {\checkmark};
    \node[label] at (1.5,0.5) {\checkmark};
    \node[label] at (1.5,1.5) {$\times$};
\end{tikzpicture}
 &
\begin{tikzpicture}[
    x=0.54cm,
    y=0.54cm,
    cell/.style={draw=black!55, line width=0.3pt},
    label/.style={font=\scriptsize}
]
    \fill[green!18] (0,0) rectangle (1,1);
    \fill[black!10] (1,0) rectangle (2,1);
    \fill[black!10] (0,1) rectangle (1,2);
    \fill[red!16] (1,1) rectangle (2,2);
    \draw[cell] (0,0) grid (2,2);
    \draw[very thick, dashed, black!70] (0,0) rectangle (1,1);
    \node[label] at (0.5,-0.26) {$0$};
    \node[label] at (1.5,-0.26) {$1$};
    \node[label] at (1,-0.82) {$p_2$};
    \node[label, anchor=east] at (-0.12,0.5) {$0$};
    \node[label, anchor=east] at (-0.12,1.5) {$1$};
    \node[label, rotate=90] at (-0.82,1) {$p_1$};
    \node[font=\large, text=black!35] at (0.5,1.5) {$\star$};
    \node[label] at (0.5,0.5) {\checkmark};
    \node[label] at (1.5,0.5) {\checkmark};
    \node[label] at (1.5,1.5) {$\times$};
\end{tikzpicture}
 &
\begin{tikzpicture}[
    x=0.54cm,
    y=0.54cm,
    cell/.style={draw=black!55, line width=0.3pt},
    label/.style={font=\scriptsize}
]
    \fill[black!10] (0,0) rectangle (1,1);
    \fill[black!10] (1,0) rectangle (2,1);
    \fill[green!18] (0,1) rectangle (1,2);
    \fill[red!16] (1,1) rectangle (2,2);
    \draw[cell] (0,0) grid (2,2);
    \draw[very thick, definitionorange] (0,1) rectangle (1,2);
    \node[label] at (0.5,-0.26) {$0$};
    \node[label] at (1.5,-0.26) {$1$};
    \node[label] at (1,-0.82) {$p_2$};
    \node[label, anchor=east] at (-0.12,0.5) {$0$};
    \node[label, anchor=east] at (-0.12,1.5) {$1$};
    \node[label, rotate=90] at (-0.82,1) {$p_1$};
    \node[font=\large, text=definitionorange] at (0.5,1.5) {$\star$};
    \node[label] at (0.5,0.5) {\checkmark};
    \node[label] at (1.5,0.5) {\checkmark};
    \node[label] at (1.5,1.5) {$\times$};
\end{tikzpicture}
 \\[-0.1em]
\begin{minipage}[t]{0.26\linewidth}
\centering\scriptsize
$\mathsf{Safe}(x)\defeq\{(0,0),(1,0),(0,1)\}$.
\end{minipage}
&
\begin{minipage}[t]{0.26\linewidth}
\centering\scriptsize
$\mathsf{Safe}_1(x)\times \mathsf{Safe}_2(x)=\{(0,0)\}$.
\end{minipage}
&
\begin{minipage}[t]{0.26\linewidth}
\centering\scriptsize
$\varphi_2:\mathsf{Always}(p_2=0)$ permits $1\in S_1^C(x)$.
\end{minipage}
\end{tabular}
\end{adjustbox}
\endgroup

  \caption{A $2 \times 2$ safe-coordination instance in which unilateral local masks lose the safe optimal action, while a refined obligation recovers it through teammate commitment.}
  \label{fig:cartesian_safety_loses_coordination}
\end{figure}

\begin{example}\label{ex:cartesian_safety_loses_coordination}
Figure~\ref{fig:cartesian_safety_loses_coordination} gives a two-agent safe-coordination instance. At state $x$, each agent chooses an action $p_i\in\{0,1\}$, and safety excludes only the simultaneous choice of~$1$:
$
\mathsf{Safe}(x)\defeq\{(0,0),(1,0),(0,1)\}.
$
Suppose the team reward favours $(1,0)$. A central shield may admit the whole safe relation and therefore preserve this coordinated optimum. A purely unilateral decomposition, which permits an action only when it is safe against every teammate action, keeps only $\{0\}\times\{0\}$: $p_1=1$ is safe only when $p_2=0$, and $p_2=1$ is safe only when $p_1=0$. Thus unilateral (factorised) masking loses the optimal safe action. A certified local obligation, for example $\mathsf{Always}(p_2=0)$, records the missing commitment: under that contract, agent~1 can again be permitted to choose $p_1=1$, recovering $(1,0)$ without admitting the unsafe pair $(1,1)$.
\end{example}

This paper studies compositional decentralised shielding which allows the preservation of optimal safe coordination actions for cooperative MARL. We construct local shields from certified local temporal contracts, while preserving a single shared global safety objective. We ask whether decentralised shields can recover optimal team-return safe coordination without falling back to a conservative Cartesian decomposition of the central shield. The technical challenge is to let agents rely on one another's obligations without introducing an unsound proof cycle. We answer it by certifying whole contract tuples through a circular fixed-point construction and projecting the jointly certified action relation into decentralised runtime masks. Learning then selects among finite tuples of local temporal obligations whose conjunction entails the single globally shared temporal objective, exposing high-return coordinated behaviours whilst upholding deterministic safety end-to-end.

\paragraph*{Contributions.} 
Our contributions are as follows:
\mbox{}\\
\begin{enumerate*}[label=(\roman*)]
  \item We show how a shared global safety specification in $\SafeLTL$ can be decomposed through candidate local obligations and compiled into deterministic safety automata for shield synthesis.\newline
  \item We provide formal soundness guarantees for decentralised shielding based on circular assume-guarantee reasoning, local $\SafeLTL$ contracts, and contract-product fixed points. Moreover, we also prove recoverability of the optimal safe deterministic policy under explicit library-representability conditions. \newline 
  \item We show how bounded certified profile search and discounted upper-confidence contract selection can reduce conservatism while preserving certification.\newline 
  \item Furthermore, we evaluate reward and safety performance across 6 cooperative multi-agent environments using an extensive set of 15 algorithmic variants.
\end{enumerate*}

\section{Related Work} \label{sec:related_work}
\paragraph*{Multi-Agent Shielding.} 
First covered in~\cite{elsayed2021safe}, multi-agent shielding has since been studied in both centralised and factorised (decentralised) forms. There is also the notion of \emph{dynamic shielding}~\cite{xiao2023model}, which allows shields to split, merge, and recompute based on the current agent configuration, for example when interaction is local in space. In particular, Alshiekh et al.~\cite{alshiekh2018safe} compile safety specifications into automata and synthesise shields that are correct by construction and \emph{minimally interfering}, in the sense that they forbid an action only when allowing it could lead to a violation. We adapt the same automata-theoretic safety-game perspective as~\cite{alshiekh2018safe,kupferman2001model} for decentralised contracts.

\paragraph*{Compositional Shielding.} 
Melcer et al.~\cite{melcer2022shield} address shield decentralisation by compiling a centralised shield into per-agent local shields, which is a \emph{shield decomposition} method or Cartesian factorisation. Their construction starts from a known-safe joint action for each shield-state/observation pair and incrementally enlarges each agent's locally permitted action set only when every induced joint action remains safe and the successor shield state remains locally identifiable. This preserves safety guarantees in a communication-free setting. The price of this decentralisation is that the safe joint-action relation must admit a Cartesian-style decomposition into independent local permissions. Jointly safe but non-factorisable coordinated actions are excluded, making the decentralised shield strictly more conservative than the original centralised shield in general.

Brorholt et al.~\cite{brorholt2024compositional} show that naively projecting a centralised shield can lose coordination information, and recover sound decentralised shields through \emph{ordered assume-guarantee} reasoning. Their method requires an acyclic dependency structure: an agent may rely only on guarantees earlier in the order, and the associated cascading-learning procedure exploits this order to obtain in-distribution training and Pareto-style guarantees under local-optimality assumptions. This is the closest precedent to our setting. Our contribution differs in both the proof principle and the learning style. Rather than fixing an acyclic dependency order, we certify an entire tuple of local obligations simultaneously by a contract-product fixed point. This permits circular coordination assumptions and therefore does not require cascading training or an acyclic training order. Moreover, algorithmic synthesis of local assume-guarantee properties is left unaddressed; we keep a single shared global $\SafeLTL$ objective and search over finite tuples of local $\SafeLTL$ obligations whose conjunction entails it. Learning then selects among pre-certified contracts that can expose high-return safe behaviours while preserving deterministic safety.

\section{Preliminaries} \label{sec:preliminaries}
\paragraph*{Multi-Agent Reinforcement Learning.} \label{sec:sg}
We model multi-agent RL problems as an $n$-agent \emph{concurrent stochastic game} (a.k.a.~a {\em Markov game}~\cite{littman1994markov})
$\mathcal{M} \defeq (\mathcal{S}, s_0, \mathcal{A}^1, \dots, \mathcal{A}^n, d^1,\dots,d^n, \mathcal{P}, \mathcal{R}^1, \dots, \mathcal{R}^n, \gamma)$, where $\mathcal{S}$ is the finite {\em state space} and $s_0 \in \mathcal{S}$ is the {\em initial state}; $\Ag\defeq\{1,\dots,n\}$ is the set of {\em agents}; each $\mathcal{A}^i$ is the finite {\em action space} of agent $i \in \Ag$ and $\mathcal{A}\defeq\prod_{i\in\Ag}\mathcal{A}^i$ is the {\em joint action space};
each $d^i:\mathcal{S}\to 2^{\mathcal{A}^i}\setminus\{\emptyset\}$ is the {\em action-availability protocol} of agent~$i$, with {\em legal  joint-action} set $d(s)\defeq\prod_{i\in\Ag}d^i(s)$ at state $s$, so $a$ is {\em legal} at $s$ iff $a\in d(s)$.
If the environment has no state-dependent feasibility restrictions, then $d^i(s)=\mathcal{A}^i$ for all $i$ and $s$;
$\mathsf{Dom}(d)\defeq\{(s,a)\in\mathcal{S}\times\mathcal{A}\mid a\in d(s)\}$ is the {\em legal state-action domain}; $\mathcal{P}: \mathsf{Dom}(d)\to \Delta(\mathcal{S})$ is the {\em transition function}, with $\Delta(\mathcal{S})$ being the set of {\em probability distributions} over $\mathcal{S}$; each  $\mathcal{R}^i: \mathsf{Dom}(d)\times \mathcal{S}\to \mathbb{R}$ is the (Markovian) {\em reward function} of agent $i$; and $\gamma$ is the {\em discount factor}. The action protocol captures state-dependent primitive-action feasibility before shielding. 

A (stochastic) \emph{policy} $\pi^i$ (also known as a {\em strategy}) of agent $i$ is a function $\pi^i : \mathcal{S} \to \Delta(\mathcal{A}^i)$ such that $\pi^i(a^i\mid s)=0$ whenever $a^i\notin d^i(s)$. A {\em joint policy} $\pi=(\pi^1,\dots,\pi^n)$ induces the decentralised product distribution
$
\pi(a\mid s)\defeq\prod_{i=1}^{n}\pi^i(a^i\mid s),
$
where
$
a=(a^1,\dots,a^n).
$
The joint policy of other agents is $\pi^{-i}(a^{-i}|s) \defeq \prod_{j \neq i}\pi^j(a^j|s)$, where $a^{-i}$ is the joint action except agent $i$'s action.   In the \emph{cooperative} strategic setting, the agents are evaluated by a common team objective. We use the standard team-return criterion:
\[
J_{s_0}(\pi)\defeq\mathbb{E}^{\pi}\!\left[\sum_{t=0}^{\infty}\gamma^t\sum_{i=1}^n \mathcal{R}^i(s_t,a_t,s_{t+1})\right],
\]
where $\pi=(\pi^1,\dots,\pi^n)$ is the joint policy. A joint policy $\pi^\star$ is \emph{team-optimal} iff $\pi^\star \in \arg\max_{\pi} J_{s_0}(\pi)$, that is, no other joint policy achieves higher expected discounted total reward from the initial state.

Throughout, for a distribution $\mu\in\Delta(Z)$, write
$\Supp(\mu)\defeq\{z\in Z\mid \mu(z)>0\}$ for its support, i.e.  the set of values that can occur with non-zero probability. An {\em execution} from $s$ is an alternating sequence $s_0a_0s_1a_1\cdots$ with $s_0=s$, $a_t\in d(s_t)$, and $s_{t+1}\in\Supp(\mathcal{P}(s_t,a_t))$ at every step; finite execution prefixes are defined similarly. The execution is consistent with $\pi$ if $\pi(a_t\mid s_t)>0$ at every step. We write $\IPath^\pi(s)$ for the set of state paths $s_0s_1\cdots$ induced by executions from $s$ that are consistent with $\pi$. The same joint policy induces a finite {\em Markov chain}~\cite{baier2008principles} $\mathcal{M}^{\pi}\defeq(\mathcal{S},s_0,\mathcal{P}^{\pi})$, where
$
\mathcal{P}^{\pi}(s,s')
\defeq
\sum_{a\in d(s)} \pi(a\mid s)\,\mathcal{P}(s,a)(s').
$

\paragraph*{Safe LTL.}
We express the desired global behaviour as a trace property in the safety fragment of \emph{Linear Temporal Logic}, denoted \SafeLTL. We consider a \emph{labelled stochastic game} obtained by extending a Markov game $\mathcal{M}$ with a finite global alphabet $\Ap$ and a global labelling function
$\mathcal{L}:\mathcal{S}\to 2^{\Ap}$. The shared safety objective is interpreted over this global labelling. For decentralised contracts we also fix local alphabets $\Ap_i\subseteq\Ap$ which specify the global propositions that agent~$i$'s local obligations may mention. Along a path, local observations are obtained by projecting the global labels to $\Ap_i$.
The syntax of the safety fragment is generated by the following grammar:
$$
\begin{array}{rcl}
\phi &::=& \true \mid \false \mid p \mid \neg p \mid \phi \wedge \phi \mid \phi \vee \phi \mid \Next \phi \mid \phi \WeakUntil \phi,
\end{array}
$$
where $p \in \Ap$, and $\true,\false$ denote the Boolean constants. We read $\Next$ as {\em next} and $\WeakUntil$ as {\em weak until}: $\phi\WeakUntil\psi$ requires $\phi$ until $\psi$, or $\phi$ forever if $\psi$ never occurs. Thus the fragment captures invariance-style properties: negation is restricted to atoms and temporal progression is expressed through next and weak until. This presentation is equivalent to the standard negation-normal Safety LTL fragment using $\Next$, $\Release$, and $\WeakUntil$~\cite{sistla1994safety}, since $\phi\Release\psi$ can be written as $\psi\WeakUntil(\phi\wedge\psi)$. We use the always operator $\Always \phi \defeq \phi \WeakUntil \false$. The connectives $\rightarrow$ and $\leftrightarrow$ are used only as standard abbreviations with their expansions normalised to the safety grammar, since negation is only allowed on atomic propositions.

Let $\Trace\defeq(2^{\Ap})^\omega$ be the set of traces over $\Ap$. Each path $\rho=s_0s_1s_2\cdots \in \IPath^\pi(s)$ induces the global trace $\Trace_{\mathcal{L}}(\rho)\defeq\mathcal{L}(s_0)\,\mathcal{L}(s_1)\,\mathcal{L}(s_2)\cdots \in \Trace$ and, for each agent $i$, the local trace $\Trace_{\mathcal{L}}(\rho)\!\upharpoonright_{\Ap_i}\defeq(\mathcal{L}(s_0)\cap\Ap_i)\,(\mathcal{L}(s_1)\cap\Ap_i)\,(\mathcal{L}(s_2)\cap\Ap_i)\cdots\in(2^{\Ap_i})^\omega$. We write $(\tau,k)\models \phi$ when formula $\phi$ holds at position $k$ of trace $\tau$, and $\tau\models\phi$ for satisfaction at position $0$. The inductive trace semantics are provided in Appendix~\ref{app:safe_ltl_semantics} as standard. In the main text we use the induced-chain satisfaction relation
$
\mathcal{M}^{\pi}\models\phi
\quad\text{iff}\quad
\Trace_{\mathcal{L}}(\rho)\models \phi
\text{ for every }\rho\in\IPath^\pi(s_0).
$
Thus satisfaction is universal over paths: a policy satisfies the safety formula only when all possible executions that it can generate satisfy the formula.

\paragraph{Problem Statement.} \label{sec:problem_statement}
In the labelled cooperative stochastic game defined above, fix a shared safety specification $\Phi_{\mathsf{safe}}\in\SafeLTL(\Ap)$. For a joint policy $\pi=(\pi^1,\dots,\pi^n)$, the safe policy set is
$
\Pi_{\mathsf{safe}}
\defeq
\left\{
\pi \;\middle|\; \mathcal{M}^{\pi}\models \Phi_{\mathsf{safe}}
\right\}.
$
The formal target is to find a 
team-return-optimal safe joint policy, that is:
$$
\pi^\star \in \arg\max_{\pi\in \Pi_{\mathsf{safe}}} J_{s_0}(\pi).
$$
Thus the objective is not merely to rule out unsafe executions: among policies whose induced traces all satisfy the temporal safety property, the learner should preserve as much discounted team return as possible. As in Example~\ref{ex:cartesian_safety_loses_coordination}, the optimum may require coordinated joint actions that a Cartesian decomposition of the central safe-action relation would discard. We target decentralised execution: each shield restricts one agent to locally admissible actions, while certification of the active contract ensures every joint execution admitted by the distributed masks satisfies $\Phi_{\mathsf{safe}}$.

\section{Contract Shielding} \label{sec:shield_synthesis}
Contract shield synthesis constructs decentralised action masks whose Cartesian product is safe under a globally certified contract. Rather than deploying the central safety-game action relation directly, the construction compiles candidate local obligations into automata and certifies local action rectangles through a contract-product fixed point. We use a finite \emph{safety abstraction}, obtainable by bisimulation~\cite{el2021abstraction}, that is exact for the contract shielding problem. Each contract supplies one temporal obligation per agent. Using the finite bad-prefix characterisation of safety properties~\cite{kupferman2001model}, each local $\SafeLTL$ obligation induces a finite deterministic safety monitor: its bad states recognise exactly the finite prefixes from which the obligation is irrecoverably violated. The shield is synthesised from their product with the environment and certifying non-reachability of unsafe states; the resulting fixed point certifies decentralised action interfaces as opposed to a central joint-action controller.

Figure~\ref{fig:tldr_pipeline} summarises our approach. The labelled game and global safety specification are compiled into contract-specific safety products with local obligations providing the assume-guarantee interface of Section~\ref{sec:ag_reasoning}; resulting contract products are solved and projected to decentralised masks in Section~\ref{sec:shield_construction}. The certified contracts then form the library used for learning-time selection in Section~\ref{sec:contract_updates}.

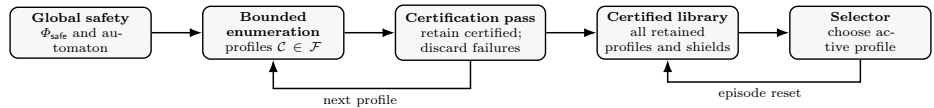
\begin{figure}[H]
  \centering
  \begin{adjustbox}{max width=\linewidth}
\begin{tikzpicture}[
    >=Latex,
    line width=0.75pt,
    node distance=0.95cm,
    box/.style={
        draw,
        rounded corners=1.5mm,
        fill=black!3,
        minimum height=0.96cm,
        text width=2.45cm,
        align=center,
        font=\scriptsize
    },
    arrow/.style={->, thick}
]
    \node[box] (spec) {\textbf{Global safety}\\$\Phi_{\mathsf{safe}}$ and automaton};
    \node[box, right=of spec] (search) {\textbf{Bounded enumeration}\\profiles $\mathcal{C}\in\mathcal{F}$};
    \node[box, right=of search, text width=2.62cm, minimum height=1.04cm] (fixed)
        {\textbf{Certification pass}\\retain certified; discard failures};
    \node[box, right=of fixed] (library) {\textbf{Certified library}\\all retained profiles and shields};
    \node[box, right=of library] (select) {\textbf{Selector}\\choose active profile};

    \draw[arrow] (spec) -- (search);
    \draw[arrow] (search) -- (fixed);
    \draw[arrow] (fixed) -- (library);
    \draw[arrow] (library) -- (select);
    \draw[arrow] (fixed.south) -- ++(0,-0.52) -|
        node[pos=0.28, below, font=\scriptsize] {next profile} (search.south);
    \draw[arrow] (select.south) -- ++(0,-0.42) -|
        node[pos=0.26, below, font=\scriptsize] {episode reset} (library.south);
\end{tikzpicture}
\end{adjustbox}
  \caption{High-level pipeline of contract-based compositional shielding.}
  \label{fig:tldr_pipeline}
\end{figure}

\subsection{Assume-Guarantee Contracts} \label{sec:ag_reasoning}
Assume-guarantee reasoning is a technique from the toolkit of \emph{compositional verification}. In the multi-agent setting we use it to control how much coordination information each local shield may rely on. We write $\AG{A}{C}{G}$ to mean that component~$C$ guarantees~$G$ under assumption~$A$. A local contract can express more than ``avoid a violation'': it can also encode coordination promises that make globally safe and high-return behaviour available under decentralised execution.

We begin by defining the local action rectangles used as shield interfaces:
\begin{definition}[Local Action Rectangle]\label{def:local_action_rectangle}
At an environment state $s\in\mathcal{S}$, a local action rectangle (or simply rectangle) is a non-empty Cartesian product
\[
R\defeq R^1\times\cdots\times R^n
\qquad\text{with}\qquad
\emptyset\neq R^i\subseteq d^i(s)
\quad\text{for every }i\in\Ag.
\]
Thus $R\subseteq d(s)$, and each factor $R^i$ is the local action mask exposed to agent~$i$.
\end{definition}

We write traces over the global alphabet as words in $(2^{\Ap})^\omega$, finite prefixes as words in $(2^{\Ap})^\ast$, and agent-$i$ observations as their pointwise projections to $(2^{\Ap_i})^\omega$. Next, we introduce local obligations and contracts.

\begin{definition}[Local Obligation]\label{def:local_obligation}
A local \SafeLTL{} obligation for agent $i$ is a formula $\varphi_i\in\SafeLTL(\Ap_i)$ over that agent's local alphabet.
\end{definition}
\begin{definition}[Contract]\label{def:contract}
A \SafeLTL{} contract is a tuple of local obligations, one for each agent,
\[
\mathcal{C}\defeq\langle \varphi_i\rangle_{i=1}^n
\in \prod_{i=1}^{n}\SafeLTL(\Ap_i).
\]
\end{definition}

A contract can certify safety only when the conjunction of its local obligations entails the shared global safety objective. Together, Definitions~\ref{def:local_obligation} and~\ref{def:contract}
fix the local syntactic objects used below. Their semantics is lifted to global traces by projection: for each component, define
$
\llbracket \varphi_i \rrbracket
\defeq
\{\tau\in\Trace \mid \tau\!\upharpoonright_{\Ap_i}\models \varphi_i\},
$
where $\tau\!\upharpoonright_{\Ap_i}$ is the pointwise projection of $\tau$ to
$\Ap_i$. For a contract $\mathcal{C}=\langle\varphi_i\rangle_{i=1}^n$, write
$
\llbracket \mathcal{C} \rrbracket \defeq \bigcap_{i=1}^{n}\llbracket \varphi_i \rrbracket .
$
Equivalently, for every global trace $\tau\in\Trace$,
$
\tau\in\llbracket \mathcal{C} \rrbracket
\quad\Longleftrightarrow\quad
\forall i\in\{1,\dots,n\},\ \tau\!\upharpoonright_{\Ap_i}\models\varphi_i.
$
This lets local obligations be interpreted on global traces by ignoring propositions outside the corresponding agent alphabet. For a global formula
$\Psi\in\SafeLTL(\Ap)$, write
$\mathcal{C}\models_{\!\uparrow}\Psi$ when
$ \forall \tau\in\Trace.
\left(\forall i\in\Ag,\ \tau\!\upharpoonright_{\Ap_i}\models\varphi_i\right)
\Longrightarrow
\tau\models\Psi
$. 
Semantic entailment involving local obligations is always understood through this lifted interpretation. Thus a contract is a decomposition of the global requirement into local promises.

\begin{definition}[Finite Contract Search Space]\label{def:finite_contract_search_space}
Fix a formula-depth bound $D$ and, for each agent $i$, a finite local formula language $\mathsf{Form}_i^{\le D}\subseteq\SafeLTL(\Ap_i)$. The finite contract search space is
$
\mathcal{C}_D
\defeq
\prod_{i=1}^{n}\mathsf{Form}_i^{\le D}.
$
\end{definition}

\begin{definition}[Circular Contract Assumptions]\label{def:circular_contract_assumptions}
For a candidate contract $\mathcal{C}=\langle\varphi_i\rangle_{i=1}^n$, the local assume-guarantee obligation of agent $i$ is
$
\AG{A_i^{\mathcal{C}}}{i}{\varphi_i},
$
where
$
A_i^{\mathcal{C}} \defeq \bigwedge_{j\neq i}\varphi_j.
$
The conjunction in $A_i^{\mathcal{C}}$ is semantic notation using the lifted
interpretation above: it denotes the intersection
$\bigcap_{j\neq i}\llbracket\varphi_j\rrbracket$ over global traces, and may
therefore combine obligations whose syntax ranges over different local
alphabets. When the index set is empty, the conjunction is $\true$.
\end{definition}
Thus agent $i$ may rely on the other agents satisfying their local obligations, while it must satisfy its own obligation. This is circular assume-guarantee reasoning: the obligations in $\mathcal{C}$ justify one another as a simultaneous contract rather than through a fixed linear order. Certification below checks the whole tuple at once: each agent may rely on the other obligations only because the fixed point proves that the projected shields preserve all obligations simultaneously, which enables sound circularity. The automata used for this certification are defined next.

\begin{definition}[Contract Safety Automata]\label{def:contract_safety_automata}
For a contract $\mathcal{C}=\langle\varphi_i\rangle_{i=1}^n$, each local obligation $\varphi_i$ is compiled into a contract safety automaton
$
\mathcal{A}_i^{\mathcal{C}}\defeq(Q_i,q_i^0,\Sigma_i,\delta_i,B_i),
$
where $Q_i$ is a finite set of states, $q_i^0\in Q_i$ is the initial state, $\Sigma_i\defeq 2^{\Ap_i}$, $\delta_i:Q_i\times\Sigma_i\to Q_i$ is the total transition function, and $B_i\subseteq Q_i$ is an absorbing set of bad states. Writing $\delta_i^\ast$ for the finite-word extension of $\delta_i$, the finite words that reach $B_i$ are exactly the bad prefixes of $\varphi_i$:
\[
\delta_i^\ast(q_i^0,w)\in B_i
\quad\Longleftrightarrow\quad
\forall \tau_i\in\Sigma_i^\omega.\; w\tau_i\not\models\varphi_i .
\]
Thus a local trace satisfies $\varphi_i$ precisely when no finite prefix reaches $B_i$. For compactness, write $\ell_i(s)\defeq\mathcal{L}(s)\cap\Ap_i$ for the projection of the global label at $s$ to agent~$i$'s alphabet.
\end{definition}

\begin{definition}[Contract Product]\label{def:contract_product}
For a candidate contract $\mathcal{C}=\langle\varphi_i\rangle_{i=1}^n$ with contract
safety automata $\mathcal{A}_i^{\mathcal{C}}$ from
Definition~\ref{def:contract_safety_automata}, the contract product has state space
$
X^{\mathcal{C}}\defeq\mathcal{S}\times Q_1\times\cdots\times Q_n.
$
A product state is written $x^{\mathcal{C}}=(s,q_1,\dots,q_n)$, where $q_i$ is the current automaton state for agent $i$'s obligation. The initial contract product state is
$
x_0^{\mathcal{C}}
\defeq
(s_0,\delta_1(q_1^0,\ell_1(s_0)),\dots,\delta_n(q_n^0,\ell_n(s_0))).
$
A contract product state is locally safe if $q_i\notin B_i$ for every agent. For a legal joint action $a=(a^1,\dots,a^n)\in d(s)$, the qualitative successor set is
\[
\mathrm{Post}_{\mathcal{C}}(x^{\mathcal{C}},a)
\defeq
\left\{
(s',\delta_1(q_1,\ell_1(s')),\dots,\delta_n(q_n,\ell_n(s')))
\;\middle|\;
s'\in\Supp(\mathcal{P}(s,a))
\right\}.
\]
A contract-product execution from $x_0^{\mathcal{C}}$ is an alternating sequence
$
\langle x_0^{\mathcal{C}}a_0x_1^{\mathcal{C}}a_1\cdots \rangle
$
such that, writing $x_t^{\mathcal{C}}=(s_t,\bar q_t)$, for every $t\ge 0$,
\[
a_t\in d(s_t)
\qquad\text{and}\qquad
x_{t+1}^{\mathcal{C}}\in\mathrm{Post}_{\mathcal{C}}(x_t^{\mathcal{C}},a_t).
\]
\end{definition}
The contract product in Definition~\ref{def:contract_product} is the natural Markov state space for shielded controllers.
In decentralised learning, agent~$i$ uses the active contract identifier together with an augmentation sufficient to recover its projected mask.

\subsection{Contract Product Shields} \label{sec:shield_construction}
We now turn a candidate contract into actual local shields. For a fixed $\mathcal{C}$, the shield construction computes one shared assume-guarantee fixed point over $X^{\mathcal{C}}$ and then projects the resulting action relation to each agent.

\begin{definition}[Contract Predecessor]\label{def:contract_predecessor}
At a contract product state $x^{\mathcal{C}}=(s,q_1,\dots,q_n)$, let $R(x^{\mathcal{C}})$ range over
local action rectangles at $s$ in the sense of
Definition~\ref{def:local_action_rectangle}.
For $W\subseteq X^{\mathcal{C}}$, such a rectangle is $W$-closed at $x^{\mathcal{C}}$ if
\[
\mathrm{Post}_{\mathcal{C}}(x^{\mathcal{C}},a)\subseteq W
\text{ for every } a\in R(x^{\mathcal{C}}).
\]
Let $\mathsf{Rect}_{\mathcal{C}}(W,x^{\mathcal{C}})$ be the set of $W$-closed local action rectangles at $x^{\mathcal{C}}$.
The assume-guarantee predecessor operator is
\[
\mathcal{T}_{\mathcal{C}}(W)
\defeq
\left\{
x^{\mathcal{C}}\in X^{\mathcal{C}} \;\middle|\;
\begin{array}{@{}l@{}}
x^{\mathcal{C}}\text{ is locally safe and}\\
\mathsf{Rect}_{\mathcal{C}}(W,x^{\mathcal{C}})\neq\emptyset
\end{array}
\right\}.
\]
\end{definition}

\begin{definition}[Rectangular Winning Region]\label{def:rectangular_winning_region}
The rectangular winning region of the contract product is the greatest fixed
point
$
\WinningRegion_{\mathcal{C}} \defeq \nu W.\,\mathcal{T}_{\mathcal{C}}(W).
$
It is the largest locally safe region from which the shield can offer a
non-empty Cartesian action interface closed under all induced successors. For
$x^{\mathcal{C}}\in \WinningRegion_{\mathcal{C}}$, write
$
\mathsf{Rect}_{\mathcal{C}}(x^{\mathcal{C}})\defeq\mathsf{Rect}_{\mathcal{C}}(\WinningRegion_{\mathcal{C}},x^{\mathcal{C}})
$
for the winning rectangles available at $x^{\mathcal{C}}$.
\end{definition}

The operator $\mathcal{T}_{\mathcal{C}}$ is monotone. Since $X^{\mathcal{C}}$ is
finite, the descending iteration $W_0=X^{\mathcal{C}}$ and
$W_{k+1}=\mathcal{T}_{\mathcal{C}}(W_k)$ stabilises at
$\WinningRegion_{\mathcal{C}}$, so
$\WinningRegion_{\mathcal{C}}=\mathcal{T}_{\mathcal{C}}(\WinningRegion_{\mathcal{C}})$.
This fixed point is the certificate used below.

\begin{definition}[Certified Contract]\label{def:certified_contract}
A candidate contract $\mathcal{C}=\langle\varphi_i\rangle_{i=1}^n\in\mathcal{C}_D$ is \emph{certified} if:
\begin{enumerate}
    \item the local obligations entail the global safety objective under lifted
    semantics, $\mathcal{C}\models_{\!\uparrow}\Phi_{\mathsf{safe}}$;
    \item the initial contract product state $x_0^{\mathcal{C}}$ lies in $\WinningRegion_{\mathcal{C}}$.
\end{enumerate}
\end{definition}

The entailment condition makes the local promises sufficient for global safety. The winning-region condition makes them operational: from the initial state, a non-empty decentralised interface can keep the active contract outside all bad states. Thus certification fails only when the promises are too weak or not enforceable by a rectangular shield from the start.
This is the initial-state condition used by Algorithm~\ref{alg:initial_contract}. For every $x^{\mathcal{C}}\in \WinningRegion_{\mathcal{C}}$, choose any certified rectangle $R_{\mathcal{C}}(x^{\mathcal{C}})\in\mathsf{Rect}_{\mathcal{C}}(x^{\mathcal{C}})$. This membership is the closure certificate: every joint action in $R_{\mathcal{C}}(x^{\mathcal{C}})$ is legal and has all contract-product successors in $\WinningRegion_{\mathcal{C}}$. Where multiple rectangles satisfy this property, we use a deterministic tie-breaker which orders choices lexicographically maximising $(\sum_i |R^i|,\prod_i |R^i|,(R^1,\dots,R^n))$, with each $R^i$ ordered by the local action order. The safety results hold in generality, for any choice satisfying the closure certificate. The tie-breaker allows fixing a reproducible shield.
\begin{definition}[Projected Local Shield]\label{def:projected_local_shield}
For a certified contract $\mathcal{C}$ and product state $x^{\mathcal{C}}\in\WinningRegion_{\mathcal{C}}$, the
deployed local shield for agent $i$ is the projection
$
\ContractShieldMask{i}(x^{\mathcal{C}}) \defeq R_{\mathcal{C}}^i(x^{\mathcal{C}}).
$
The shield is defined only on $\WinningRegion_{\mathcal{C}}$.
\end{definition}

\begin{definition}[Local Mask Identifiability]\label{def:local_mask_identifiability}
For a certified contract $\mathcal{C}$, the projected shield is locally mask-identifiable for agent $i$ if some function of the active contract identifier and agent-$i$'s observation history returns $\ContractShieldMask{i}(x_t^{\mathcal{C}})$ for every reachable contract-product execution and time $t$.
\end{definition}
Definition~\ref{def:projected_local_shield} fixes the local mask exposed to agent~$i$: the $i$th factor of a certified rectangle. Definition~\ref{def:local_mask_identifiability} is the corresponding information condition for decentralised execution; it requires that this factor be computable from the active contract identifier and the information available to agent~$i$
(for example, its augmented local contract-product state), without necessitating agent~$i$ to reconstruct the other agents' masks.

\begin{definition}[Admitted Joint Actions]\label{def:admitted_joint_actions}
For a certified contract $\mathcal{C}$ and product state $x^{\mathcal{C}}\in\WinningRegion_{\mathcal{C}}$, the
joint-action relation admitted by the deployed local shields is
$
\mathsf{Adm}_{\mathcal{C}}(x^{\mathcal{C}})
\defeq
\prod_{i\in\Ag}\ContractShieldMask{i}(x^{\mathcal{C}}).
$
\end{definition}
\begin{restatable}[Circular Compositional Soundness\RestatedTitleSuffix]{theorem}{circularCompositionalSoundnessTheorem}\label{thm:circular_compositional_soundness}
Let $\mathcal{C}=\langle\varphi_i\rangle_{i=1}^n$ be a certified contract. For any contract-product execution
$x_0^{\mathcal{C}}a_0x_1^{\mathcal{C}}a_1\cdots$ starting from $x_0^{\mathcal{C}}$, if $a_t\in\mathsf{Adm}_{\mathcal{C}}(x_t^{\mathcal{C}})$ at every time $t$, then every reachable contract product state remains in $\WinningRegion_{\mathcal{C}}$. Consequently, the local obligations hold along the execution, and the induced global trace satisfies $\Phi_{\mathsf{safe}}$.
\end{restatable}
The proof is available in Appendix~\ref{appsec:circular_compositional_soundness}.
Informally, the theorem says that a certified circular contract behaves like a joint safety proof split across local masks. The invariant is simple: start inside the winning region, and every projected shield action keeps the next product state inside it. Thus, once the rectangle has been certified, the agents do not need any coordination — any joint action admitted by the local masks is still inside the certified successor region.

\section{Learning with Contract Shields} \label{sec:contract_updates}
Building on Section~\ref{sec:shield_synthesis}, this section gives the concrete pipeline: bounded offline construction of a certified library in Section~\ref{subsec:bounded_certified_profile_search}, its permissiveness and optimality interpretation in Section~\ref{subsec:permissiveness_metrics}, and online contract selection in Section~\ref{subsec:selector_bandit}.

\subsection{Library Construction: Bounded Certified Profile Search}\label{subsec:bounded_certified_profile_search}
The finite contract search space $\mathcal{C}_D$ is finite, with candidates $\mathcal{C}$ drawn from the finite search space of
Definition~\ref{def:finite_contract_search_space}, so we enumerate candidate local-obligation tuples before learning and certify each tuple using the fixed-point construction above. For any finite ordered candidate family $\mathcal{F}\subseteq\mathcal{C}_D$, the certified library induced by $\mathcal{F}$ is
$
\mathcal{L}_{\mathsf{cert}}(\mathcal{F})
\defeq
\{\mathcal{C}\in\mathcal{F}\mid \mathcal{C}\ \text{is certified}\},
$
with the order inherited from $\mathcal{F}$. Since $\mathcal{F}$ is finite, so is $\mathcal{L}_{\mathsf{cert}}(\mathcal{F})$; moreover, membership in the library is exactly certification, so every retained contract satisfies the hypotheses of Theorem~\ref{thm:circular_compositional_soundness}. The selector
therefore uses this catalogue of safe shield configurations without constructing or validating new safety certificates online. The algorithm enumerates scoped candidate subsets in a deterministic prioritised order subject to the configured profile limit and formula depth limit; this order uses a temporal-form heuristic that prefers persistent $G\psi$ obligations before next-state $X\psi$ obligations. Algorithm~\ref{alg:initial_contract} summarises this finite preprocessing stage. Return-optimal behaviour does not necessarily arise from the first certified profile. A stricter contract can enable solving the decentralised learner's coordination problem, so the learning-time selector optimises over the retained library.

\begin{algorithm}[t]
\DontPrintSemicolon
\caption{Certified profile library construction}
\label{alg:initial_contract}
\KwIn{ordered bounded profile family $\mathcal{F}\subseteq\mathcal{C}_D$, initial abstract states $S_0$, global safety formula $\Phi_{\mathsf{safe}}$}
\KwOut{initial certified profile $\mathcal{C}_0$ and certified library $\mathcal{L}_{\mathsf{cert}}$}
$\mathcal{L}_{\mathsf{cert}}\gets\emptyset$\;
\ForEach{$\mathcal{C}\defeq\langle\varphi_i\rangle_{i=1}^n \in \mathcal{F}$ in search order}{
  \If{$\mathcal{C}\not\models_{\!\uparrow}\Phi_{\mathsf{safe}}$}{
    \textbf{continue}\;
  }
  compile the local obligation automata $\mathcal{A}_1^{\mathcal{C}},\dots,\mathcal{A}_n^{\mathcal{C}}$\;
  construct the contract product $X^{\mathcal{C}}$ and greatest fixed point $\WinningRegion_{\mathcal{C}}$\;
  \If{every initial contract state induced by $S_0$ lies in $\WinningRegion_{\mathcal{C}}$}{
    derive the local shields from the certified rectangles $R_{\mathcal{C}}(x^{\mathcal{C}})$\;
    add $\mathcal{C}$ to $\mathcal{L}_{\mathsf{cert}}$\;
  }
}
\If{$\mathcal{L}_{\mathsf{cert}}=\emptyset$}{
  \Return failure\;
}
$\mathcal{C}_0\gets$ the first profile in $\mathcal{L}_{\mathsf{cert}}$ in certification order\;
\Return $\mathcal{C}_0$ and $\mathcal{L}_{\mathsf{cert}}$\;
\end{algorithm}

\subsection{Permissiveness and Optimality}\label{subsec:permissiveness_metrics}

Definition~\ref{def:cardinality_permissiveness} gives a statewise mask-size metric, obtained simply by counting the sizes of local action masks relative to the protocol-available action sets.
\begin{definition}[Cardinality Permissiveness]\label{def:cardinality_permissiveness}
At a shield state $x^{\mathcal{C}}=(s,\bar q)$, the cardinality permissiveness of
agent~$i$'s local mask is
$
\rho_i(x^{\mathcal{C}})
\defeq
\frac{|\ContractShieldMask{i}(x^{\mathcal{C}})|}
     {|d^i(s)|}.
$
\end{definition}
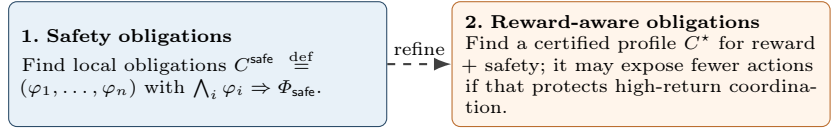
\begin{figure}[t]
  \centering
  \definecolor{permissivenessquestionblue}{RGB}{50,88,122}
\definecolor{permissivenessquestionbluelight}{RGB}{232,241,248}
\definecolor{permissivenessdefinitionorange}{RGB}{220,128,54}
\definecolor{permissivenessdefinitionorangelight}{RGB}{255,245,235}

\begin{adjustbox}{max width=\linewidth}
\begin{tikzpicture}[
    >=Latex,
    node distance=0.85cm,
    problem/.style={
        draw,
        rounded corners=1.1mm,
        minimum height=1.65cm,
        text width=4.65cm,
        align=left,
        font=\scriptsize,
        inner sep=5pt
    },
    safety/.style={problem, fill=permissivenessquestionbluelight, draw=permissivenessquestionblue!85!black},
    reward/.style={problem, fill=permissivenessdefinitionorangelight, draw=permissivenessdefinitionorange!85!black},
    arrow/.style={->, thick, draw=black!75}
]
    \node[safety] (safeproblem) {
        \textbf{1. Safety obligations}\\
        Find local obligations
        $C^{\mathsf{safe}}\defeq(\varphi_1,\ldots,\varphi_n)$
        with
        $\bigwedge_i\varphi_i\Rightarrow\Phi_{\mathsf{safe}}$.
    };

    \node[reward, right=0.85cm of safeproblem] (rewardproblem) {
        \textbf{2. Reward-aware obligations}\\
        Find a certified profile $C^\star$ for
        reward $+$ safety; it may expose fewer actions if that protects
        high-return coordination.
    };

    \draw[arrow, dashed] (safeproblem.east) -- node[above, font=\scriptsize] {refine} (rewardproblem.west);
\end{tikzpicture}
\end{adjustbox}
  \caption{Safety certification and reward-aware contract selection are distinct objectives.}
  \label{fig:permissiveness_objective}
\end{figure}

Notice cardinality permissiveness is not the optimisation objective: $\rho_i(x^{\mathcal{C}})=1$ means no protocol-available primitive action of agent~$i$ is removed at $x^{\mathcal{C}}$, while smaller values indicate more intervention. A large mask can preserve many low-value actions while excluding coordinated high-return behaviours, and a stricter certified profile can make useful coordination easier to learn. Contract selection therefore optimises observed team return over certified profiles, not total mask cardinality. The relevant completeness question is whether some certified contract preserves the globally optimal safe behaviour. We formalise this by comparing policies admitted by certified decentralised shields.

\begin{definition}[Contract-Compatible Policies]\label{def:contract_compatible_policy}
Let $\mathcal{C}$ be a certified contract. For a joint policy $\pi$, let $\mathsf{Reach}_{\mathcal{C}}(\pi)$ be the least subset of $\WinningRegion_{\mathcal{C}}$ containing
$x_0^{\mathcal{C}}$ and closed under admitted $\pi$-successors: if
$x^{\mathcal{C}}=(s,\bar q)\in\mathsf{Reach}_{\mathcal{C}}(\pi)$,
$a\in\Supp(\pi(\cdot\mid s))\cap\mathsf{Adm}_{\mathcal{C}}(x^{\mathcal{C}})$, and
$x'\in\mathrm{Post}_{\mathcal{C}}(x^{\mathcal{C}},a)$, then
$x'\in\mathsf{Reach}_{\mathcal{C}}(\pi)$. The policies compatible with $\mathcal{C}$ are
\[
\Pi(\mathcal{C})
\defeq
\left\{
\pi \;\middle|\;
\forall x^{\mathcal{C}}=(s,\bar q)\in\mathsf{Reach}_{\mathcal{C}}(\pi),\
\Supp(\pi(\cdot\mid s))\subseteq\mathsf{Adm}_{\mathcal{C}}(x^{\mathcal{C}})
\right\}.
\]
\end{definition}

For a deterministic Markov policy $\pi$, write $a_\pi(s)$ for its unique joint action at state $s$. Then $\pi\in\Pi(\mathcal{C})$ iff, for every
$x^{\mathcal{C}}=(s,\bar q)\in\mathsf{Reach}_{\mathcal{C}}(\pi)$, we have
$
a_\pi(s)\in\mathsf{Adm}_{\mathcal{C}}(x^{\mathcal{C}}).
$
Equivalently, each local action $a^i_\pi(s)$ lies in the deployed local mask
$\ContractShieldMask{i}(x^{\mathcal{C}})$ at every reachable contract-product state.

\begin{restatable}[Optimal Safe Recovery\RestatedTitleSuffix]{theorem}{optimalSafeRecoveryTheorem}\label{thm:optimal_safe_recovery}
Let $\mathcal{L}_{\mathsf{cert}}$ be a certified library, and let
$\pi^\star\in\arg\max_{\pi\in\Pi_{\mathsf{safe}}}J_{s_0}(\pi)$ be a
deterministic Markov policy. If some $\mathcal{C}^\star\in\mathcal{L}_{\mathsf{cert}}$ satisfies $\pi^\star\in\Pi(\mathcal{C}^\star)$, then
$
\max\left\{
J_{s_0}(\pi)
\;\middle|\;
\mathcal{C}\in\mathcal{L}_{\mathsf{cert}},\ \pi\in\Pi(\mathcal{C})
\right\}
=
J_{s_0}(\pi^\star).
$
Equivalently, when empty inner feasible sets are ignored,
\[
\max_{\mathcal{C}\in\mathcal{L}_{\mathsf{cert}}}\max_{\pi\in\Pi(\mathcal{C})}J_{s_0}(\pi)
=J_{s_0}(\pi^\star).
\]
\end{restatable}
The proof is available in Appendix~\ref{appsec:optimal_safe_recovery}. In particular, suppose $\pi^\star$ is a deterministic Markov optimum in $\Pi_{\mathsf{safe}}$, the local alphabets are action-separating (e.g. by folding the last action into the state or using a non-Markovian labelling function) for $\pi^\star$, and the depth bound $D$ and candidate family are rich enough that $\mathcal{L}_{\mathsf{cert}}$ contains a certified contract $\mathcal{C}^\star$ with
\[
\ContractShieldMask{i}(x^{\mathcal{C}^\star})=\{a^i_{\pi^\star}(s)\}
\qquad
\text{for every reachable }x^{\mathcal{C}^\star}=(s,\bar q)\text{ and every }i\in\Ag,
\]
then $\pi^\star\in\Pi(\mathcal{C}^\star)$, so Theorem~\ref{thm:optimal_safe_recovery} applies. The local masks also enforce the globally optimal safe joint action on all reachable states. These statements concern deterministic Markov optima. For the finite discounted, fully observable cooperative concurrent stochastic games considered in our setting, policy randomisation is not required to attain the optimal team value: an optimal deterministic Markov joint policy exists, although stochastic policies remain useful during learning~\cite{puterman2014markov,sutton1998reinforcement}.

\subsection{Bandit Selector} \label{subsec:selector_bandit}
The learning-time selector is an outer loop over the certified library: it chooses which pre-certified contract (shield) governs each training block, while the inner MARL algorithm continues updating its policies under the currently active shields. Each certified contract is treated as an arm. Pulling arm $\mathcal{C}$ means training for one dwell block under $\mathcal{C}$'s shield and measuring the resulting team return. The selector is parameterised by:
\begin{itemize}[leftmargin=*,nosep]
  \item the finite certified library $\mathcal{L}_{\mathsf{cert}}$ of candidate contracts;
  \item the initial certified contract $\mathcal{C}_0\in\mathcal{L}_{\mathsf{cert}}$;
  \item the warm-up horizon $H_0$, measured in completed episodes;
  \item the dwell length $H$, measured in completed episodes per arm pull;
  \item the discount factor $\eta\in(0,1]$ for historical block statistics; and
  \item the exploration coefficient $\beta\ge 0$.
\end{itemize}
Let $E_b$ be the completed episodes in a dwell block $b$ run under
contract $\mathcal{C}_b$, and let $R_e^i$ be the episode return of agent~$i$ in episode~$e$. The selector uses the normalised team-return score
$
y_b
\defeq
\overline{R}_b
\defeq
\frac{1}{|E_b|\,n}\sum_{e\in E_b}\sum_{i=1}^n R_e^i,
$
which is equivalent to the usual team return up to the constant factor $n$. The resulting bandit is \emph{non-stationary}: an arm's payoff changes as the inner policies continue to learn under the active shields. Discounting therefore estimates recent block performance rather than a fixed stationary arm mean, following discounted UCB methods for non-stationary bandits~\cite{kocsis2006discounted,garivier2011upper}. For each certified contract $\mathcal{C}$, maintain discounted count and score statistics $N_{\mathcal{C}}$ and $S_{\mathcal{C}}$. When a block under $\mathcal{C}_b$ completes, all arms are first discounted by setting $N_{\mathcal{C}}\leftarrow\eta N_{\mathcal{C}}$ and $S_{\mathcal{C}}\leftarrow\eta S_{\mathcal{C}}$ for every $\mathcal{C}\in\mathcal{L}_{\mathsf{cert}}$, and the observed arm is then updated by
\[
N_{\mathcal{C}_b} \leftarrow N_{\mathcal{C}_b} + 1,\qquad
S_{\mathcal{C}_b} \leftarrow S_{\mathcal{C}_b} + y_b,\qquad
\widehat{\mu}_{\mathcal{C}'} \defeq S_{\mathcal{C}'}/N_{\mathcal{C}'},
\]
where $\widehat{\mu}_{\mathcal{C}'}$ is defined for each $\mathcal{C}'$ with $N_{\mathcal{C}'}>0$. After
$H_0$ completed warm-up episodes under $\mathcal{C}_0$, the selector accumulates dwell blocks of $H$ completed episodes, tries certified contracts whose identifiers have not yet been visited, then choosing
\[
\mathcal{C}_{b+1}
\in
\arg\max_{\substack{\mathcal{C}\in\mathcal{L}_{\mathsf{cert}}\\N_{\mathcal{C}}>0}}
\left(
\widehat{\mu}_{\mathcal{C}}
+
\beta
\sqrt{
\frac{\log(1+\max\{1,\sum_{\mathcal{C}'}N_{\mathcal{C}'}\})}{N_{\mathcal{C}}}
}
\right),
\]
where $\beta\ge 0$ is the exploration coefficient. Ties are resolved deterministically, preferring to keep the current contract and otherwise using the certified-library order. Algorithm~\ref{alg:contract-update} in Appendix~\ref{app:technical_remarks} gives the discounted-UCB update.

\begin{restatable}[Safety of Certified Selection\RestatedTitleSuffix]{theorem}{certifiedSelectionSafetyTheorem}\label{thm:certified_selection_safety}
Let $\mathcal{L}_{\mathsf{cert}}$ be a certified library. Consider any sequence of episodes in which each episode is assigned a contract $\mathcal{C}\in\mathcal{L}_{\mathsf{cert}}$, starts from an initial product state covered by $\mathcal{C}$'s certification, the shield for $\mathcal{C}$ is deployed for the whole episode, every generated joint action lies in $\mathsf{Adm}_{\mathcal{C}}(x_t^{\mathcal{C}})$ at the current contract product state, and contract switches occur only at episode reset boundaries. Then no generated episode prefix is a bad prefix of the global safety objective. Under the standard per-episode infinite-trace semantics, every episode satisfies $\Phi_{\mathsf{safe}}$.
\end{restatable}
The proof is available in Appendix~\ref{appsec:certified_selection_safety}. Theorem~\ref{thm:certified_selection_safety} allows the selector to adapt reward estimates and switch contracts between episodes, but it never admits an uncertified tuple. Thus the outer loop can optimise team return over the finite space of certified contracts while preserving deterministic safety. Theorem~\ref{thm:certified_selection_safety} does not by itself give a convergence guarantee for the discounted bandit selector; the optimal safe team return is attainable only insofar as it is represented by some certified contract in the library and the learning dynamics select the corresponding profile. Without that representability assumption, performance is limited by the expressive power of the contract language and certified library.

\section{Empirical Evaluation} \label{sec:empirical_evaluation}
We report the experimental configuration, benchmark suite, certified search counts, and learning curves. Appendix~\ref{app:environments} gives environment details, labelling, and safety objectives; Appendix~\ref{app:extended_experiments} reports full curves. The implementation is open-source and available at \url{https://github.com/nightly/contract-shielding}.

\paragraph*{Implementation Details and Hyperparameters.} The formal development uses a single global alphabet $\Ap$ and global labelling function $\mathcal{L}$. For tractability, the implementation stores per-agent sub-alphabets $\Ap_i\subseteq\Ap$ and compiles each local automaton over $\Sigma_i=2^{\Ap_i}$; automaton~$i$ receives label $\mathcal{L}(s)\cap\Ap_i$ at state $s$. Certification checks consistency with the global model: each $\Ap_i\subseteq\Ap$, each local label is the corresponding projection, and propositions shared by multiple local alphabets agree, see Definition~\ref{def:contract}. We bound candidate formula depth by $D=2$ (interpreting atoms as $D=1$). Other hyperparameters are provided in Appendix~\ref{app:hyperparameters}.

\paragraph*{Algorithms.} We compare independent learners (\algo{IPPO}~\cite{de2020independent}, \algo{IQL}~\cite{tan1993multi}), centralised-critic or value-decomposition learners (\algo{MAPPO}~\cite{yu2022surprising},
\algo{Joint-PPO}, \algo{PQN-VDN}~\cite{sunehag2017value,gallici2025simplifying}), Lagrangian variants of \algo{IPPO} and \algo{IQL}, \algo{ICPO}~\cite{achiam2017constrained}, ordinary shielded variants, and
\algo{Contract-IPPO}/\algo{Contract-IQL}. \algo{MAPPO} uses a centralised critic with decentralised actors; \algo{Joint-PPO} and \algo{Shielded-Joint-PPO} use centralised actors and critics and serve as upper-end reward baselines although outside our decentralised-execution setting.

\paragraph*{Environments.} Evaluation covered six cooperative MARL benchmarks: \env{Flatland}~\cite{mohanty2020flatland}, \env{Connector}~\cite{bonnet2024jumanji}, \env{Level-Based Foraging}~\cite{christianos2020shared}, \env{RWARE}~\cite{papoudakis2020benchmarking}, \env{Pressure Plate}~\cite{ahmed2022deep}, and \env{Car Platoon}~\cite{brorholt2024compositional}. Example~\ref{example:plate} illustrates one safety property and contract; Appendix~\ref{app:environments} expands the rest.

\begin{example}\label{example:plate}
In \env{Pressure Plate}, 2 holder agents open doors by standing on assigned plates while a runner crosses to the goal. Let $Z$ be goal achievement; let $I_k$, $A_k$, and $B_k$ mean that the runner is in door~$k$, after door~$k$, and able to return before door~$k$; and let $H_k$ mean that holder~$k$ has not abandoned plate~$k$ before the goal. The holder obligations let the runner's mask rely on held doors when admitting crossings; factorised masks must be robust to holder departures, so they conservatively keep the runner out of the protected door regions. The default return-route objective and one certified local contract are
\[
\begin{aligned}
\Phi^{\mathsf{pp}}_{\mathsf{safe}}
&\defeq \Always\!\bigwedge_{k=0}^{1}
\bigl((\neg Z\wedge(I_k\vee A_k))\rightarrow B_k\bigr),\\[-0.2ex]
\mathcal{C}^{\mathsf{pp}}
&\defeq \langle \Always H_0,\ \Always H_1,\ \Phi^{\mathsf{pp}}_{\mathsf{safe}}\rangle .
\end{aligned}
\]
\end{example}

\newcommand{\MissingResultGraphic}{%
  \fbox{\parbox[c][0.18\linewidth][c]{0.88\linewidth}{\centering Result pending}}%
}

\newcommand{\ResultGraphic}[1]{%
  \IfFileExists{#1.png}{\includegraphics[width=\linewidth]{#1.png}}{\MissingResultGraphic}%
}

\newcommand{\PaperResultGraphic}[1]{%
  \ResultGraphic{#1}%
}

\newcommand{\ExperimentConciseLegend}{%
  \includegraphics[width=0.84\linewidth]{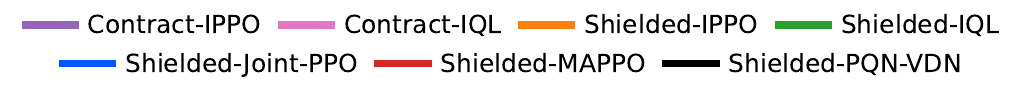}%
  \par\vspace{0.2ex}
}

\newcommand{\ExperimentFullLegend}{%
  \includegraphics[width=0.98\linewidth]{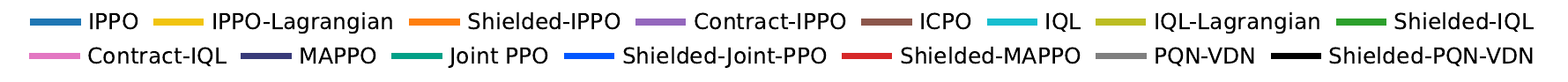}%
  \par\smallskip
}

\newcommand{\RewardPanel}[2]{%
  \begin{minipage}[t]{0.28\linewidth}
    \vspace{0pt}%
    \centering
    \PaperResultGraphic{figures/experiments/#2/shielded/papers/reward}\\[-0.4ex]
    {\scriptsize #1}%
  \end{minipage}%
}

\newcommand{\EnvironmentResultRow}[2]{%
  \par\smallskip
  {\small\bfseries #1}\par\vspace{-0.8ex}
  \makebox[\linewidth][c]{%
    \begin{minipage}[t]{0.5\linewidth}
      \vspace{0pt}%
      \centering
      \PaperResultGraphic{#2/all/papers/reward}\\[-0.6ex]
      {\scriptsize Reward mean}%
    \end{minipage}%
    \begin{minipage}[t]{0.5\linewidth}
      \vspace{0pt}%
      \centering
      \PaperResultGraphic{#2/all/papers/safety}\\[-0.6ex]
      {\scriptsize Cumulative safety violations}%
    \end{minipage}%
  }%
}

\paragraph*{Analysis.}
Figure~\ref{fig:main_reward_curves} shows the shielded and contract-aware reward curves. The plotted quantity is reward: safety is enforced structurally by certified masks and Theorems~\ref{thm:circular_compositional_soundness} and~\ref{thm:certified_selection_safety}; the curves measure retained return under certified actions. Where ordinary factorised shields are deployable, contract shielding improves reward and converges toward \algo{Shielded-Joint-PPO}. For \env{RWARE}, \env{Pressure Plate}, and \env{Connector}, the selected global specifications admit no factorised shield because masks must be conservative over other agents' actions. For \env{Pressure Plate}, shielded baselines use conservative local runner masks rather than a direct factorisation of $\Phi^{\mathsf{pp}}_{\mathsf{safe}}$ (Appendix~\ref{app:env_pressure_plate}). In the largest environment instances, bounded search considered up to 12,001 profiles and certified up to 390 as safe; Table~\ref{tab:contract_synthesis_counts} in Appendix~\ref{app:extended_experiments} gives all objectives and results.

\begin{figure}[t]
  \centering
  \ExperimentConciseLegend

  \begin{tabular*}{\linewidth}{@{\extracolsep{\fill}}ccc@{}}
    \RewardPanel{\env{Pressure Plate}}{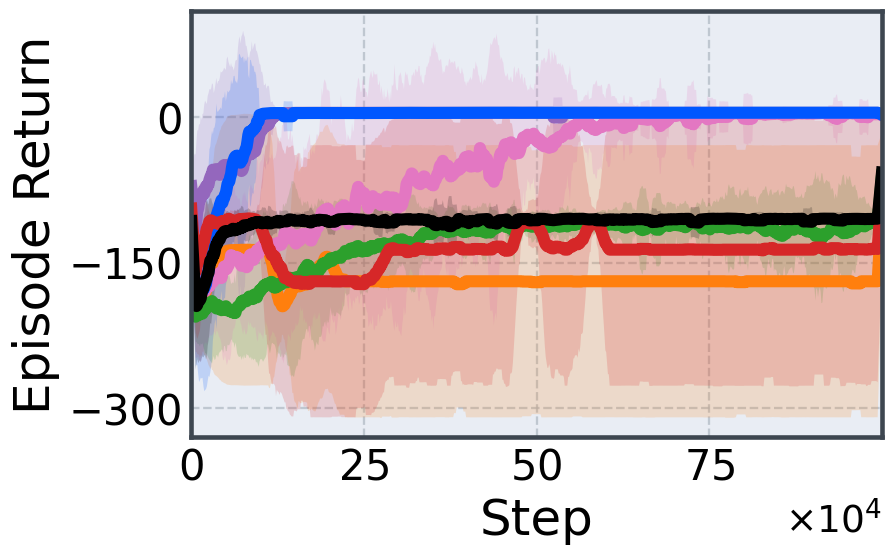} &
    \RewardPanel{\env{RWARE}}{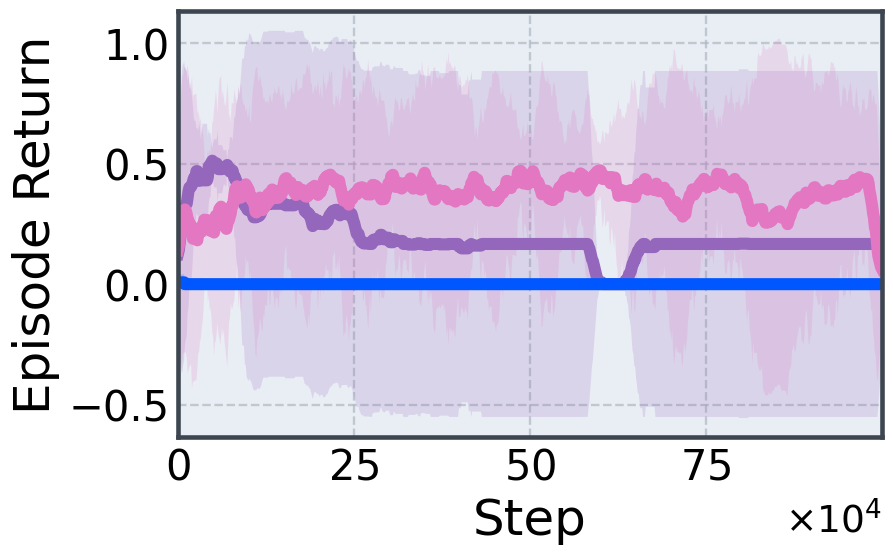} &
    \RewardPanel{\env{Connector}}{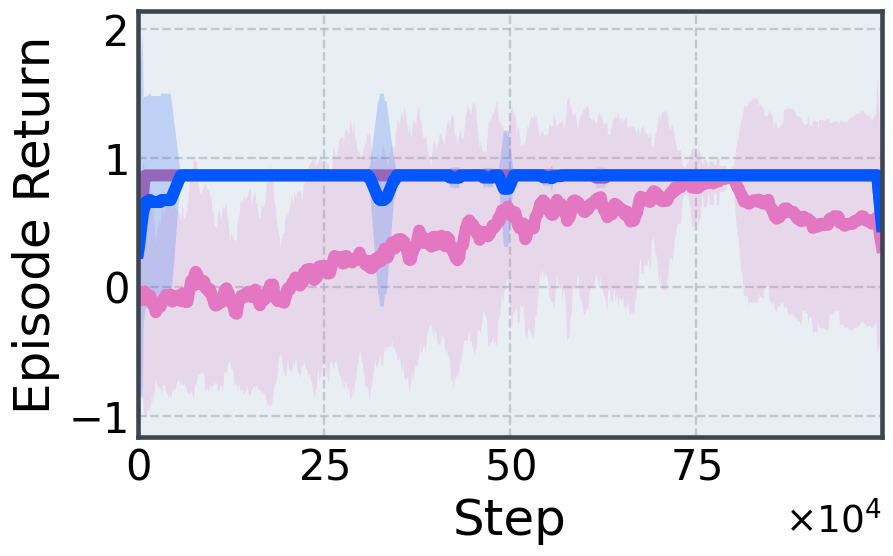} \\[-0.2ex]
    \RewardPanel{\env{Flatland}}{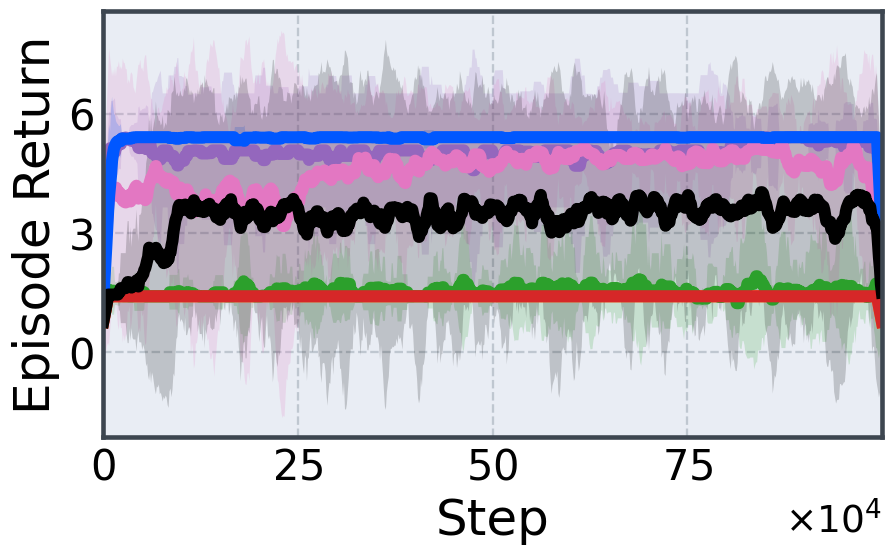} &
    \RewardPanel{\env{Level-Based Foraging}}{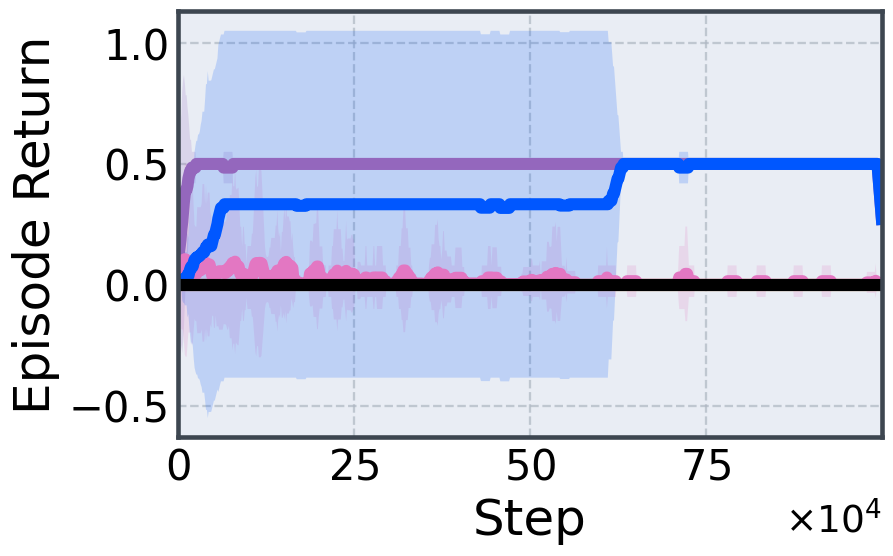} &
    \RewardPanel{\env{Car Platoon}}{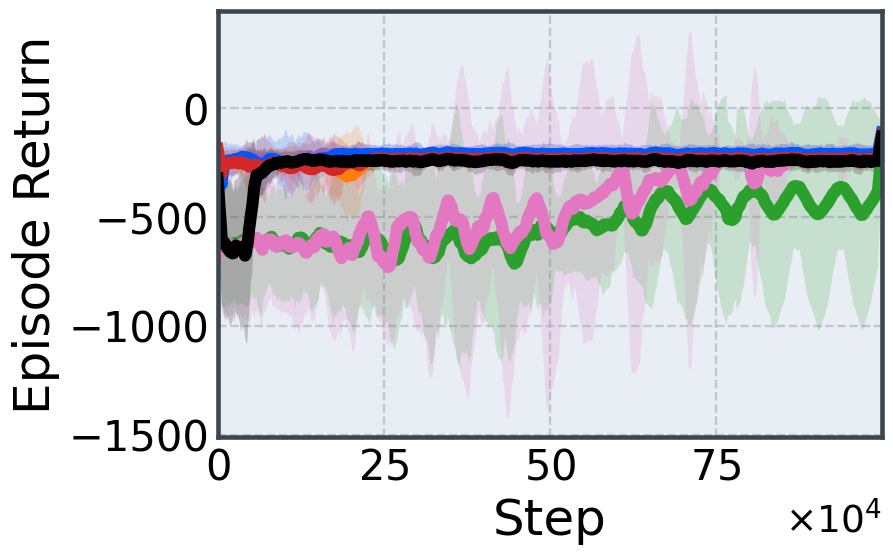}
  \end{tabular*}

  \caption{Team-mean reward curves for all six benchmarks using shielded algorithms. Shaded regions are $\pm$ 95\% confidence intervals across runs.}
  \label{fig:main_reward_curves}
\end{figure}

\section{Conclusion and Future Work} \label{sec:conclusion}
We presented contract-based decentralised shielding for cooperative MARL. Finite profiles of local $\SafeLTL$ obligations are certified by a circular assume-guarantee fixed point, projected into local action masks, and selected
during learning only from pre-certified libraries, separating reward adaptation from deterministic safety. Agents may rely on certified coordination promises while preserving one global safety objective. Across six benchmarks,
contract-aware shields recover coordinated behaviours while retaining pre-emptive safety guarantees. Future work will study quantitative and probabilistic obligations, and heuristics for finding near-reward-optimal contracts without exhaustive enumeration.

\begin{credits}
\subsubsection{\ackname} The research described in this paper was partially supported by the EPSRC (grant number EP/X015823/1) and the UKRI Centre for Doctoral Training in Safe and Trusted AI (grant number EP/S0233356/1).

\subsubsection{\discintname}
 The authors have no competing interests to declare that are
relevant to the content of this article.
\end{credits}

\bibliographystyle{splncs04}
\bibliography{cite}


\appendix

\section*{Appendix}
The appendix is organised as follows:
\begin{itemize}[label={},leftmargin=0pt,itemindent=0pt,itemsep=0pt]
    \item Appendix~\ref{app:safe_ltl_semantics}: $\SafeLTL$ semantics are provided.
    \item Appendix~\ref{appsec:technical_results}: Technical results and proofs are provided.
    \item Appendix~\ref{app:environments}: Environment-level benchmark details are provided.
    \item Appendix~\ref{app:extended_experiments}: Complete experimental results are reported.
    \item Appendix~\ref{app:technical_remarks}: Technical remarks on contract selection are provided.
    \item Appendix~\ref{app:hyperparameters}: Hyperparameters are provided for replication.
\end{itemize}

\section{Safe LTL Semantics} \label{app:safe_ltl_semantics}

Let $\Trace\defeq(2^{\Ap})^\omega$ be the set of traces over $\Ap$. Then, we can define the inductive semantics used in Section~\ref{sec:preliminaries} as follows.
\logicsemanticssetrhswidth{0.58\linewidth}
\begin{logicsemantics}[For any trace $\tau \in \Trace$, position $k \in \mathbb{N}$, and formulae $\phi,\psi$: ]
\logicsemanticssection{Trace semantics}
\logicsemanticsrule{$(\tau,k)\models \true$}{$\iff$}{always}
\logicsemanticsrule{$(\tau,k)\models \false$}{$\iff$}{never}
\logicsemanticsrule{$(\tau,k)\models p$}{$\iff$}{$p\in \tau[k]$ for $p\in\Ap$}
\logicsemanticsrule{$(\tau,k)\models \neg p$}{$\iff$}{$p\notin \tau[k]$ for $p\in\Ap$}
\logicsemanticsrule{$(\tau,k)\models \phi\wedge \psi$}{$\iff$}{$(\tau,k)\models \phi$ and $(\tau,k)\models \psi$}
\logicsemanticsrule{$(\tau,k)\models \phi\vee \psi$}{$\iff$}{$(\tau,k)\models \phi$ or $(\tau,k)\models \psi$}
\logicsemanticsrule{$(\tau,k)\models \Next \phi$}{$\iff$}{$(\tau,k+1)\models \phi$}
\logicsemanticsruleblock{$(\tau,k)\models \phi \WeakUntil \psi$}{$\iff$}{$\big(\forall m\ge k.\ (\tau,m)\models \phi\big)$ or\\ $\big(\exists j\ge k.\ (\tau,j)\models \psi\ \text{and}$\\ $\forall m\in\{k,\dots,j-1\}.\ (\tau,m)\models \phi\big)$}
\logicsemanticsrule{$(\tau,k)\models \Always \phi$}{$\iff$}{$\forall j\ge k.\ (\tau,j)\models \phi$}
\logicsemanticssection{Model satisfaction}
\logicsemanticsrule{$\tau\models \phi$}{$\iff$}{$(\tau,0)\models \phi$}
\logicsemanticsrule{$\mathcal{M}^{\pi}\models \phi$}{$\iff$}{$\Trace_{\mathcal{L}}(\rho)\models \phi$ for every $\rho\in\IPath^\pi(s_0)$}
\end{logicsemantics}

\section{Technical Results} \label{appsec:technical_results}
This section provides all of the formal proofs of results stated in the main text.

\subsection{Proof of Theorem 1. Circular Compositional Soundness} \label{appsec:circular_compositional_soundness}

\circularCompositionalSoundnessTheorem*

\begin{proof}
We prove the invariant $x_t^{\mathcal{C}}\in \WinningRegion_{\mathcal{C}}$ by induction on time (establishing the invariant for all $t\ge 0$). The base case is exactly the second certification condition, $x_0^{\mathcal{C}}\in \WinningRegion_{\mathcal{C}}$. For the step case, assume $x_t^{\mathcal{C}}\in \WinningRegion_{\mathcal{C}}$ and write $x_t^{\mathcal{C}}=(s_t,\bar q_t)$. Since the executed joint action satisfies
$a_t\in\mathsf{Adm}_{\mathcal{C}}(x_t^{\mathcal{C}})$, Definitions~\ref{def:projected_local_shield} and~\ref{def:admitted_joint_actions} give
\[
a_t\in
\prod_{i\in\Ag}R_{\mathcal{C}}^i(x_t^{\mathcal{C}})
=
R_{\mathcal{C}}(x_t^{\mathcal{C}}).
\]
By the choice of the certified rectangle,
$R_{\mathcal{C}}(x_t^{\mathcal{C}})\in\mathsf{Rect}_{\mathcal{C}}(x_t^{\mathcal{C}})
=\mathsf{Rect}_{\mathcal{C}}(\WinningRegion_{\mathcal{C}},x_t^{\mathcal{C}})$.
Thus $R_{\mathcal{C}}(x_t^{\mathcal{C}})$ is a $\WinningRegion_{\mathcal{C}}$-closed local action rectangle at $x_t^{\mathcal{C}}$; in particular, $a_t$ is legal at $s_t$ and
\[
\mathrm{Post}_{\mathcal{C}}(x_t^{\mathcal{C}},a_t)
\subseteq \WinningRegion_{\mathcal{C}}.
\]
Every possible next product state
$x_{t+1}^{\mathcal{C}}$ is in this successor set, hence
$x_{t+1}^{\mathcal{C}}\in \WinningRegion_{\mathcal{C}}$.

Because $\WinningRegion_{\mathcal{C}}=\mathcal{T}_{\mathcal{C}}(\WinningRegion_{\mathcal{C}})$, every state in
$\WinningRegion_{\mathcal{C}}$ is locally safe. For each agent $i$ and time $t$, the automaton component of $x_t^{\mathcal{C}}$ is
\[
q_{i,t}
=
\delta_i^\ast(q_i^0,\ell_i(s_0)\ell_i(s_1)\cdots\ell_i(s_t)),
\]
by the definition of the contract product. Since $q_{i,t}\notin B_i$ for every $t$, no finite prefix of the projected local trace is a bad prefix. By Definition~\ref{def:contract_safety_automata}, the projected local trace satisfies $\varphi_i$ for every agent $i$. The lifted entailment condition $\mathcal{C}\models_{\!\uparrow}\Phi_{\mathsf{safe}}$ in
Definition~\ref{def:certified_contract} then implies that the global trace satisfies $\Phi_{\mathsf{safe}}$.
\end{proof}

\subsection{Proof of Theorem 2. Optimal Safe Recovery} \label{appsec:optimal_safe_recovery}

\optimalSafeRecoveryTheorem*

\begin{proof}
Fix any certified $\mathcal{C}\in\mathcal{L}_{\mathsf{cert}}$ and any
$\pi\in\Pi(\mathcal{C})$. We show that every contract-product execution consistent with
$\pi$ uses only admitted actions. The initial product state is in
$\mathsf{Reach}_{\mathcal{C}}(\pi)$ by definition. If
$x_t^{\mathcal{C}}=(s_t,\bar q_t)\in\mathsf{Reach}_{\mathcal{C}}(\pi)$, compatibility gives
$\Supp(\pi(\cdot\mid s_t))\subseteq\mathsf{Adm}_{\mathcal{C}}(x_t^{\mathcal{C}})$; therefore every
action chosen with positive probability by $\pi$ at $s_t$ is admitted. For any
successor $x_{t+1}^{\mathcal{C}}\in\mathrm{Post}_{\mathcal{C}}(x_t^{\mathcal{C}},a_t)$, closure of
$\mathsf{Reach}_{\mathcal{C}}(\pi)$ gives $x_{t+1}^{\mathcal{C}}\in\mathsf{Reach}_{\mathcal{C}}(\pi)$. By induction,
all actions along every execution consistent with $\pi$ are admitted by the
deployed shields.

Theorem~\ref{thm:circular_compositional_soundness} then implies that every such
execution satisfies $\Phi_{\mathsf{safe}}$, so $\pi\in\Pi_{\mathsf{safe}}$.
Thus every feasible pair $(\mathcal{C},\pi)$ in the certified-shield optimisation has
value at most
$\max_{\pi\in\Pi_{\mathsf{safe}}}J_{s_0}(\pi)=J_{s_0}(\pi^\star)$. Conversely,
by the hypothesis, the feasible pair $(\mathcal{C}^\star,\pi^\star)$ attains
$J_{s_0}(\pi^\star)$. The two inequalities give the claimed equality.
\end{proof}

\subsection{Proof of Theorem 3. Certified Selection Safety} \label{appsec:certified_selection_safety}

\certifiedSelectionSafetyTheorem*

\begin{proof}
Fix an arbitrary episode and let $\mathcal{C}$ be its assigned contract. Since
$\mathcal{C}\in\mathcal{L}_{\mathsf{cert}}$, the definition of the certified library gives
that $\mathcal{C}$ is certified. By hypothesis, the episode begins at an initial contract
product state in $\WinningRegion_{\mathcal{C}}$, and the shield for $\mathcal{C}$ is used for every
step of the episode with generated actions in $\mathsf{Adm}_{\mathcal{C}}(x_t^{\mathcal{C}})$.
Theorem~\ref{thm:circular_compositional_soundness}
therefore applies throughout the episode: all reachable contract product states
remain in $\WinningRegion_{\mathcal{C}}$, no contract safety automaton reaches a
bad state, and the global safety objective is satisfied. Hence no generated
episode prefix is a bad prefix of the global safety objective. Because switches occur only between episodes, the
same argument restarts independently for each selected contract.
\end{proof}

\section{Environments} \label{app:environments}

This section provides details for all of the environments considered in our empirical evaluation of Section~\ref{sec:empirical_evaluation}. We use finite safety abstractions for contract and shield synthesis. For the safety game, the abstraction is required to preserve the successor support relevant to the automata and the truth of all propositions appearing in the global and local obligations.

Figure~\ref{fig:appendix_environments} shows renders of the
six benchmark environments described below.

\begin{figure}[t]
  \centering
  \input{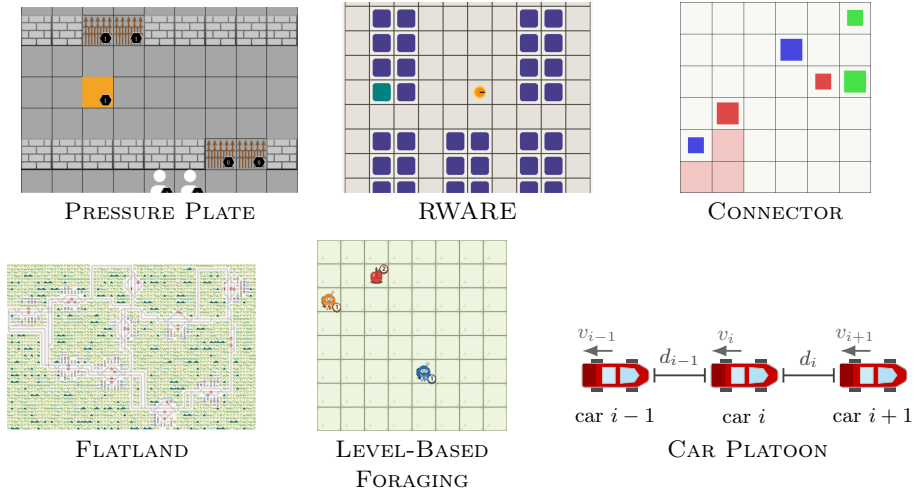}
  \caption{Renders of the six benchmark environments used in the evaluation.}
  \label{fig:appendix_environments}
\end{figure}

\subsection{\env{Pressure Plate}} \label{app:env_pressure_plate}

\env{Pressure Plate} is a role-structured grid task. Let
$G\defeq\{0,\dots,H-1\}\times\{0,\dots,W-1\}$ be the grid, let $m\defeq n-1$ be the number of doors and plates, and let agent~$n-1$ be the runner. An abstract state is
\[
s\defeq(p_0,\dots,p_{n-1},C,b_0,\dots,b_{m-1},o_0,\dots,o_{m-1},g,V_H,V_W,t),
\]
where $p_i\in G$ is an agent position, $C\subseteq\{0,\dots,m-1\}$ is a
one-step latch recording doors crossed by the runner on the preceding transition, $b_k,o_k,g\in\{0,1\}$ record whether plate~$k$ is pressed, door~$k$ is open, and the goal is reached, $V_H$ records door-holder protocol
violations, $V_W$ records runner wait-for-door violations, and $t$ is the step
count. Each agent has primitive actions
$\{\mathsf{Up},\mathsf{Down},\mathsf{Left},\mathsf{Right},\mathsf{Noop}\}$.
The transition samples a permutation of the agents, applies their proposed moves sequentially, and cancels moves into grid boundaries, walls, closed doors, or occupied cells. After movement, plate~$k$ is pressed iff the assigned holder $k$ occupies it, and door~$k$ is open iff that plate is pressed.
The latch $C'$ contains door~$k$ exactly when the runner entered a cell of door~$k$, left such a cell toward the far side, or crossed the door boundary during the transition; it is empty on non-crossing transitions. The episode terminates when some agent reaches the goal or the horizon is met.

Rewards are role-shaped: holders receive dense shaping toward their assigned
plates, while the runner is shaped toward the goal. The reported experiments
use distance normalisation $D_{\max}=10$, step cost $0.01$, and goal bonus
$10$.

For holder~$i<m$, the local contract alphabet contains its door-holding proposition, door~$i$ open and plate-pressed propositions, holder-on-plate propositions, goal achievement, and the runner's waiting, crossing, in-door, and after-door phase propositions for door~$i$. The runner's local alphabet contains, for every door, its wait-for-door, waiting, crossing, in-door,
after-door, door-open, and return-route propositions, together with goal achievement. The global alphabet contains these propositions and the finite agent-cell, agent-at-goal, agent-before-door, and plate/door status propositions needed to interpret the benchmark labels.

The default benchmark instance uses the linear layout with $H=7$, $W=4$, $n=3$, sensor range~$3$, and horizon~$150$. Let $A_k$, $I_k$, $B_k$, and $Z$ respectively abbreviate the runner-after-door, runner-in-door, runner-can-return
before door, and goal-achieved propositions. The default safety formula is
\[
\Phi^{\mathsf{pp}}_{\mathsf{safe}}
\defeq
\Always\bigwedge_{k\in\{0,1\}}
\bigl((\neg Z\wedge(I_k\vee A_k))\rightarrow B_k\bigr).
\]
The return-route label $B_k$ is computed by reachability from the runner's current cell to a cell before door~$k$, treating walls and closed doors as blocked and ignoring teammate occupancy. A direct generic per-agent shield for this formula is not realisable from the initial state, so the ordinary shielded
baseline uses the conservative runner-only formulas
\[
\begin{aligned}
\varphi_0&\defeq \true,\\
\varphi_1&\defeq \true,\\
\varphi_2&\defeq\Always\neg(I_0\vee A_0\vee I_1\vee A_1),
\end{aligned}
\]
while the global automaton still checks $\Phi^{\mathsf{pp}}_{\mathsf{safe}}$.
Certified contract profiles can instead assign holder obligations that preserve
the runner's return route while the runner carries the global return-route
obligation.

\subsection{\env{RWARE}} \label{app:env_rware}

\env{RWARE} is a warehouse queue and delivery benchmark. A state is
\[
s\defeq(u_1,\dots,u_n,\sigma_1,\dots,\sigma_q,Q,Y,t,\iota),
\]
where $u_i\defeq(x_i,y_i,d_i,\ell_i,c_i,h_i)$ records robot position, direction,
loaded status, whether the carried shelf is requested, and whether a delivery
has occurred; $\sigma_j$ records shelf identities and positions; $Q$ is the
request queue; $Y$ records standing and requested shelf locations; and $t,\iota$
are the total and inactive step counts. Each robot has actions
$\{\mathsf{NOOP},\mathsf{FORWARD},\mathsf{LEFT},\mathsf{RIGHT},
\mathsf{TOGGLE\_LOAD}\}$. Turning updates $d_i$, forward movement requests the
cell in direction $d_i$, and load toggling picks up or drops a shelf when the
warehouse interaction is legal. The transition resolves collisions using
separate robot and shelf occupancy layers; a loaded robot occupies both layers.
When a requested shelf reaches a goal cell, the request is replaced by a new
request according to the environment kernel.

Rewards are sparse and delivery based: in the default individual reward mode,
robot~$i$ receives reward only when it delivers a requested shelf, with no
movement reward or direct penalty for protocol violations.

The local contract alphabet for robot~$i$ contains the
queue-yield proposition \texttt{agent\_i\_queue\_yield\_ok}. The proposition is
false on a transition precisely when robot~$i$ is unloaded, occupies the next
forward cell of a robot carrying a requested shelf in the queue, and remains in
that blocking cell after joint-action execution and collision resolution.
The global alphabet additionally contains finite robot-position propositions,
highway and zone-membership propositions, loaded and carrying-requested status,
requested-shelf locations and at-requested-shelf facts, requested-carrier
status, blocking-requested-carrier facts, and delivery-lane and
requested-carrier-progress protocol propositions.
Writing $q_i\defeq\texttt{agent\_i\_queue\_yield\_ok}$, the default safety formula is
the per-robot queue-yield invariant
\[
\Phi^{\mathsf{rw}}_{\mathsf{safe}}
\defeq
\bigwedge_{i=0}^{n-1}\Always q_i.
\]
The default queue-conflict instance uses $n=2$, shelf columns~$3$, column
height~$8$, one shelf row, sensor range~$1$, flattened observations, request
queue size~$2$, individual rewards, no inactivity cutoff, and horizon~$200$.

Ordinary factorised shielding is not realisable for this default queue-yield
objective. The proposition $q_i$ records a joint transition fact. A factorised
shield exposes independent local action sets, so every combination of
individually permitted actions must remain safe against arbitrary teammate
choices. In the queue-conflict layout, reachable states can require the blocker
and carrier to take compatible same-step actions; no single local blocker action
is certified safe against all carrier actions, and the projected local mask can
become empty. Contract shields avoid this failure by first certifying teammate
obligations, so local masks are checked against obligation-constrained teammates
rather than the full adversarial action set.

\subsection{\env{Connector}} \label{app:env_connector}

\env{Connector} is a grid routing task. A state is
\[
s\defeq(B,p_0,\dots,p_{n-1},z_0,\dots,z_{n-1},V,t),
\]
where $B\in\{0,\dots,3n\}^{G}$ is the grid of empty cells, trails, current
positions, and targets; $p_i,z_i\in G$ are current and target cells; $V$ records
reservation violations; and $t$ is the step count. Each agent has actions
$\{\mathsf{Noop},\mathsf{Up},\mathsf{Right},\mathsf{Down},\mathsf{Left}\}$.
The action mask permits a movement only into an empty cell or the agent's own
target; illegal actions have no movement effect. Simultaneous proposed moves
are computed from the current positions. Destination conflicts are resolved by a
fixed agent-index priority rule: an agent whose proposed destination is
also proposed by a later-indexed agent is cancelled, while successful moves
leave a path marker behind, and reaching the target changes the target marker
into the agent position marker. The episode terminates when all agents are
connected or blocked, or when the horizon is met.

Dense rewards give each unconnected agent a connection bonus and a per-step
cost until it reaches its target. The default constants are
$r_{\mathsf{conn}}=1$ and $r_{\mathsf{step}}=-0.03$.

The local contract alphabet for agent~$i$ contains its aggregate
reservation-respect and route-clearance propositions, pairwise propositions
stating that agent~$i$ keeps another agent's reserved route clear and respects
that agent's reservation, and every agent's reserved-route-clear proposition.
The global alphabet contains finite current-position and target-position
propositions, action-availability propositions, connection and blocked status,
reserved-route-blocked and on-reserved-route propositions, and the aggregate and
pairwise reservation-respect and route-clearance propositions. Let
$r_i\defeq\texttt{agent\_i\_reserved\_route\_clear\_ok}$. The default safety formula is
\[
\Phi^{\mathsf{co}}_{\mathsf{safe}}
\defeq
\bigwedge_{i=0}^{n-1}\Always r_i.
\]
The default reservation-conflict experiment uses the
\texttt{reservation\_\allowbreak priority\_\allowbreak conflict} generator with fixed layout seed~$0$,
a $5\times5$ grid, $n=2$, dense rewards, the reserved cell $(2,2)$ for
agent~$0$, and horizon~$12$.

Ordinary factorised shielding is not realisable for this default
reserved-route-clearance objective. The proposition $r_i$ is owned by
agent~$i$ but can be falsified by a teammate entering or remaining on
agent~$i$'s reserved route. A local shield for the route owner must therefore
be safe against all teammate actions, yet the owner has no action that can force
the teammate off its reserved cells. Conversely, the teammate's factorised mask
cannot express the conditional commitment ``avoid the owner's reserved route
while the owner is still unconnected'' unless that commitment appears as part
of a certified contract profile. The Cartesian product of ordinary local masks
therefore cannot deploy a non-empty shield that preserves the global
route-clearance invariant from the reservation-priority layout, whereas a
contract profile can add explicit reservation-respect obligations.

\subsection{\env{Flatland}} \label{app:env_flatland}

\env{Flatland} is a railway-routing benchmark. A state is
$
s\defeq(\xi_0,\dots,\xi_{n-1},D,C,V,t),
$
where each train state
$\xi_i\defeq(p_i,d_i,\kappa_i,z_i,p_i^0,d_i^0)$ contains current position, direction,
\env{Flatland} train status, target, and initial placement; $D$ is the set of
deadlocked trains; $C$ is the set of trains stopped by conflicts; $V$ records
pairwise yield violations; and $t$ is the elapsed step count. The primitive
actions are the five \env{Flatland} rail actions: do nothing, turn left, move
forward, turn right, and stop. The transition is \env{Flatland}'s
\texttt{RailEnv} transition restricted by the rail graph, train status, malfunctions, and conflict resolution. In the single-track benchmark this transition is deterministic after the fixed scenario is generated.

Progress rewards give credit for reducing shortest-path distance to the target,
with a per-step cost, arrival bonus, and deadlock penalty. The default
constants are $c_{\mathsf{prog}}=1$, $c_{\mathsf{step}}=0.01$,
$c_{\mathsf{arr}}=5$, and $c_{\mathsf{dead}}=5$.

The local contract alphabet for train~$i$ contains all trains' deadlock
propositions and the pairwise yield propositions
\texttt{agent\_i\_yields\_to\_agent\_j\_ok} for $j\neq i$. The global alphabet
contains finite train-position propositions, target and off-map status,
\env{Flatland} train-status propositions, deadlock and conflict-stop
propositions, and pairwise yield-protocol propositions. The default safety
formula is
\[
\Phi^{\mathsf{fl}}_{\mathsf{safe}}
\defeq
\bigwedge_{i=0}^{n-1}\Always\neg(\texttt{agent\_i\_deadlocked}).
\]
The default instance is \texttt{single\_track\_meet} with $n=2$, grid
$7\times3$, unit-speed trains, tree depth~$1$, predictor depth~$10$,
progress rewards, \texttt{team\_done} termination, no malfunctions, and the
horizon is $12$. Candidate contracts can add pairwise yield obligations while preserving the global no-deadlock objective.

\subsection{\env{Level-Based Foraging}} \label{app:env_lbf}

\env{Level-Based Foraging} models cooperative loading. A state is
\[
s\defeq(p_0,\ell_0,\dots,p_{n-1},\ell_{n-1},F,V_F,V_S,V_L,V_C,t),
\]
where $p_i\in G$ and $\ell_i$ are agent positions and levels; $F$ is the finite
set of food triples $(x,y,\ell)$; $V_F,V_S,V_L,V_C$ record failed loads,
successful loads, load attempts, and cooperative-load protocol violations; and
$t$ is the step count. Each agent has actions
$\{\mathsf{NONE},\mathsf{NORTH},\mathsf{SOUTH},\mathsf{WEST},\mathsf{EAST},
\mathsf{LOAD}\}$. Invalid state-dependent actions become
$\mathsf{NONE}$, movement conflicts cancel all colliding moves, and a
$\mathsf{LOAD}$ action affects a food item only when the loader is adjacent to
that food. If the total level of adjacent loaders is below the food level, the
load fails; otherwise the food is collected.

Rewards are assigned only by loading: failed load attempts are penalised, and
successful loads give level-weighted credit to participating agents, normalised
by the total spawned food level in the default configuration.

The local contract alphabet for agent~$i$ contains
\texttt{agent\_i\_coop\_load\_ok} and every agent's failed-load proposition.
The global alphabet contains food-availability, all-food-collected,
joint-load-ready, food-position, and food-level propositions, together with
finite agent-position propositions and per-agent load-attempt, successful-load,
failed-load, cooperative-load, adjacent-food, needs-partner,
partner-adjacent-same-food, and can-load-solo propositions. The default safety
formula is
\[
\Phi^{\mathsf{lbf}}_{\mathsf{safe}}
\defeq
\bigwedge_{i=0}^{n-1}\Always\neg(\texttt{agent\_i\_failed\_load}).
\]
The default experiment uses a $5\times5$ field, $n=2$, one level-$2$ food item,
two level-$1$ agents, \texttt{force\_coop=true}, normalised rewards,
failed-load penalty~$1$, sight radius~$5$, fixed layout seed~$0$, vector
observations, and horizon~$25$.
Candidate contracts can add cooperative-load obligations using
\texttt{agent\_i\_coop\_load\_ok} while preserving the global no-failed-load
objective.

\subsection{\env{Car Platoon}} \label{app:env_car_platoon}

\env{Car Platoon} has one uncontrolled lead car and $n=N-1$ controlled followers, and is based on the environment described in~\cite{brorholt2024compositional}. Cars are ordered front to back. A state is
\[
s\defeq(v_0,\dots,v_{N-1},d_0,\dots,d_{n-1},\Delta_0,\dots,\Delta_{N-1},
V_C,V_S,t),
\]
where $v_k$ is car velocity, $d_i$ is the following gap between car~$i$ and
car~$i+1$, $\Delta_k$ records damage, $V_C,V_S$ record conservative-follow and
smooth-lead protocol violations, and $t$ is the step count. Controlled agents
choose brake, coast, or accelerate, mapped to accelerations $-2,0,2$, while the
lead car follows a velocity-dependent stochastic policy. Velocities are clipped
to $[v_{\min},v_{\max}]$, and each gap is updated by trapezoidal integration of
relative velocity:
\[
d_i' = d_i + \frac{(v_i-v_{i+1})+(v_i'-v_{i+1}')}{2}\Delta t .
\]
Contact damages the adjacent cars, and a gap is unsafe when
$d_i'\le d_{\min}$ or $d_i'\ge d_{\max}$. Rewards penalise large following
gaps and unsafe-gap violations.
The default experiment uses $N=3$, $\Delta t=1$, velocity bounds
$[v_{\min},v_{\max}]=[0,2]$, $d_{\min}=0$, $d_{\max}=20$, initial gap~$10$,
safety-violation penalty~$100$, and termination on violation.

The local contract alphabet for controlled follower~$i$ contains its gap-safe,
conservative-follow, and smooth-lead protocol propositions. The global alphabet
contains discretised velocity, front-velocity, and following-distance
propositions for each controlled follower, together with gap-safe, crashed,
too-far, near-boundary, closing/opening-fast
safe-to-accelerate-if-front-coasts, damage, and protocol propositions. Writing
$g_i\defeq\texttt{agent\_i\_gap\_safe}$, the default safety formula is
\[
\Phi^{\mathsf{cp}}_{\mathsf{safe}}
\defeq
\bigwedge_{i=0}^{n-1}\Always g_i.
\]
Candidate contracts can add conservative-follow or smooth-lead obligations
while keeping this global gap-safety objective fixed.

\section{Extended Experimental Results} \label{app:extended_experiments}

This appendix reports the extended empirical material omitted from the main paper. For each benchmark, the extended results record complete learning curves, safety-violation counts, shield-intervention statistics, and certified-contract selection traces when multiple certified profiles are available. Experiments were run on a Linux workstation with an AMD EPYC 7702P 64-core CPU and 64 logical cores and 203.48 GiB system RAM, and a single NVIDIA A16 GPU with 14.6 GiB memory.
Table~\ref{tab:contract_synthesis_counts} reports the global safety objectives, bounded-search profile counts, and additional-obligation profile counts used in the main
evaluation.

\begin{table}[t]
  \centering
  \scriptsize
  \setlength{\tabcolsep}{2pt}%
  \renewcommand{\arraystretch}{1.15}%
  \rowcolors{2}{black!3}{white}
  \begin{adjustbox}{max width=\linewidth}
  \begin{tabular}{@{} L{0.17\linewidth} L{0.63\linewidth} r r r @{}}
    \toprule
    Environment & \makecell{Safety \SafeLTL{}\\formula} & \makecell{Profiles\\considered} & \makecell{Profiles\\certified safe} & \makecell{Additional-obligation\\profiles} \\
    \midrule
    \env{Car Platoon} &
      $\displaystyle\bigwedge_{i\in\mathsf{all\_agents}}\Always\,\mathsf{agent}_i\_\mathsf{gap\_safe}$ & 12{,}001 & 22 & 1 \\
    \env{Connector} &
      $\displaystyle\bigwedge_{i\in\mathsf{all\_agents}}\Always\,\mathsf{agent}_i\_\mathsf{reserved\_route\_clear\_ok}$ & 4{,}097 & 81 & 1 \\
    \env{Flatland} &
      $\displaystyle\bigwedge_{i\in\mathsf{all\_agents}}\Always\neg\mathsf{agent}_i\_\mathsf{deadlocked}$ & 4{,}097 & 337 & 2 \\
    \env{Level-Based Foraging} &
      $\displaystyle\bigwedge_{i\in\mathsf{all\_agents}}\Always\neg\mathsf{agent}_i\_\mathsf{failed\_load}$ & 4{,}097 & 390 & 1 \\
    \env{Pressure Plate} &
      $\begin{aligned}[t]
        \Always\!\bigwedge_{k\in\{0,1\}}
        \bigl((\neg Z\wedge(I_k\vee A_k))\rightarrow B_k\bigr)
      \end{aligned}$ & 2 & 1 & 1 \\
    \env{RWARE} &
      $\displaystyle\bigwedge_{i\in\mathsf{all\_agents}}\Always\,\mathsf{agent}_i\_\mathsf{queue\_yield\_ok}$ & 1{,}369 & 16 & 0 \\
    \bottomrule
  \end{tabular}
  \end{adjustbox}
  \rowcolors{2}{}{}
  \caption{Safety objectives are shown as global \SafeLTL{} formulae alongside exported bounded-search profile counts. The final column counts retained profiles whose obligations are stricter than the global safety formula. For \env{Pressure Plate}, $Z$, $I_k$, $A_k$, and $B_k$ abbreviate goal achievement, runner in-door, runner after-door, and runner return-route propositions.
  }
  \label{tab:contract_synthesis_counts}
\end{table}

Tables~\ref{tab:training_timings_policy_gradient},
\ref{tab:training_timings_value_based}, and
\ref{tab:training_timings_shielded_contract} report the training timings.

\begin{table}[!htbp]
  \centering
  \scriptsize
  \setlength{\tabcolsep}{3pt}%
  \renewcommand{\arraystretch}{1.05}%
  \rowcolors{2}{black!3}{white}
  \begin{adjustbox}{max width=\linewidth}
  \begin{tabular}{@{} l c c c c c @{}}
    \toprule
    Environment & \algo{IPPO} & \algo{IPPO-Lagrangian} & \algo{ICPO} & \algo{MAPPO} & \algo{Joint-PPO} \\
    \midrule
    \env{Pressure Plate} & $0.81$ & $0.84$ & $0.85$ & $0.85$ & $0.88$ \\
    \env{RWARE} & $0.90$ & $0.92$ & $0.95$ & $0.88$ & $0.89$ \\
    \env{Connector} & $0.81$ & $0.84$ & $0.93$ & $0.92$ & $0.95$ \\
    \env{Flatland} & $1.04$ & $1.04$ & $1.09$ & $1.12$ & $1.14$ \\
    \env{Level-Based Foraging} & $0.73$ & $0.73$ & $0.78$ & $0.86$ & $0.82$ \\
    \env{Car Platoon} & $0.72$ & $0.73$ & $0.79$ & $0.78$ & $0.77$ \\
    \bottomrule
  \end{tabular}
  \end{adjustbox}
  \rowcolors{2}{}{}
  \caption{Policy-gradient baseline training timings. Entries are mean wall-clock hours over 3 runs.}
  \label{tab:training_timings_policy_gradient}
\end{table}

\begin{table}[!htbp]
  \centering
  \scriptsize
  \setlength{\tabcolsep}{3pt}%
  \renewcommand{\arraystretch}{1.05}%
  \rowcolors{2}{black!3}{white}
  \begin{adjustbox}{max width=\linewidth}
  \begin{tabular}{@{} l c c c @{}}
    \toprule
    Environment & \algo{IQL} & \algo{IQL-Lagrangian} & \algo{PQN-VDN} \\
    \midrule
    \env{Pressure Plate} & $0.53$ & $0.55$ & $0.67$ \\
    \env{RWARE} & $0.57$ & $0.57$ & $0.62$ \\
    \env{Connector} & $0.59$ & $0.60$ & $0.69$ \\
    \env{Flatland} & $0.82$ & $0.83$ & $0.90$ \\
    \env{Level-Based Foraging} & $0.47$ & $0.46$ & $0.60$ \\
    \env{Car Platoon} & $0.48$ & $0.48$ & $0.63$ \\
    \bottomrule
  \end{tabular}
  \end{adjustbox}
  \rowcolors{2}{}{}
  \caption{Value-based training timings. Entries are mean wall-clock hours over 3 runs.}
  \label{tab:training_timings_value_based}
\end{table}

\begin{table}[!htbp]
  \centering
  \scriptsize
  \setlength{\tabcolsep}{3pt}%
  \renewcommand{\arraystretch}{1.05}%
  \rowcolors{2}{black!3}{white}
  \begin{adjustbox}{max width=\linewidth}
  \begin{tabular}{@{} l c c c c c c c @{}}
    \toprule
    Environment & \makecell{\algo{Shielded}\\\algo{IPPO}} & \makecell{\algo{Contract}\\\algo{IPPO}} & \makecell{\algo{Shielded}\\\algo{IQL}} & \makecell{\algo{Contract}\\\algo{IQL}} & \makecell{\algo{Shielded}\\\algo{MAPPO}} & \makecell{\algo{Shielded}\\\algo{Joint-PPO}} & \makecell{\algo{Shielded}\\\algo{PQN-VDN}} \\
    \midrule
    \env{Pressure Plate} & $0.92$ & $0.85$ & $0.65$ & $0.57$ & $0.94$ & $1.10$ & $0.77$ \\
    \env{RWARE} & \textsc{n/a} & $0.96$ & \textsc{n/a} & $0.64$ & \textsc{n/a} & $1.23$ & \textsc{n/a} \\
    \env{Connector} & \textsc{n/a} & $0.88$ & \textsc{n/a} & $0.62$ & \textsc{n/a} & $1.07$ & \textsc{n/a} \\
    \env{Flatland} & $1.06$  & $1.04$ & $0.88$  & $0.85$  & $1.17$ & $1.23$ & $0.95$ \\
    \env{Level-Based Foraging} & $0.81$ & $0.80$ & $0.55$ & $0.54$ & $0.89$ & $0.92$ & $0.67$ \\
    \env{Car Platoon} & $0.77$ & $0.73$ & $0.52$ & $0.50$ & $0.83$ & $0.82$ & $0.68$ \\
    \bottomrule
  \end{tabular}
  \end{adjustbox}
  \rowcolors{2}{}{}
  \caption{Shielded and contract-aware training timings. Entries are mean wall-clock hours over 3 runs; \textsc{n/a} marks intentionally omitted baselines.}
  \label{tab:training_timings_shielded_contract}
\end{table}

\newcommand{\EnvironmentResultFigurePair}[6]{%
  \begin{figure}[!htbp]
    \centering
    \ExperimentFullLegend

    \EnvironmentResultRow{#1}{#2}
    \vspace{0.6ex}
    \EnvironmentResultRow{#4}{#5}

    \caption{Learning curves for #1 and #4. Shaded regions are 95\% confidence
    intervals across runs.}
    \label{fig:appendix_experiment_#3}
    \label{fig:appendix_experiment_#6}
  \end{figure}
}

\newcommand{\EnvironmentResultFigureSingle}[3]{%
  \begin{figure}[!htbp]
    \centering
    \ExperimentFullLegend

    \EnvironmentResultRow{#1}{#2}

    \caption{Learning curves for #1. Shaded regions are 95\% confidence
    intervals across runs.}
    \label{fig:appendix_experiment_#3}
  \end{figure}
}

\EnvironmentResultFigurePair{\env{Pressure Plate}}{figures/experiments/pressure_plate}{pressure_plate}{\env{RWARE}}{figures/experiments/rware}{rware}
\EnvironmentResultFigurePair{\env{Connector}}{figures/experiments/connector}{connector}{\env{Flatland}}{figures/experiments/flatland}{flatland}
\EnvironmentResultFigurePair{\env{Level-Based Foraging}}{figures/experiments/level_based_foraging}{level_based_foraging}{\env{Car Platoon}}{figures/experiments/car_platoon}{car_platoon}

The extended plots are reported in
Figures~\ref{fig:appendix_experiment_pressure_plate},
\ref{fig:appendix_experiment_rware},
\ref{fig:appendix_experiment_connector},
\ref{fig:appendix_experiment_flatland},
\ref{fig:appendix_experiment_level_based_foraging}, and
\ref{fig:appendix_experiment_car_platoon}.

\clearpage

\section{Technical Remarks} \label{app:technical_remarks}

\paragraph{Policies and shielded controllers.} The set
$\Pi_{\mathsf{safe}}$ in Section~\ref{sec:problem_statement} records safety over environment traces. Once a contract is fixed, the current controller is naturally Markov on the contract product $x^{\mathcal{C}}=(s,\bar q)$. However, viewed on the original environment, $\bar q$ is finite memory updated from the label
stream. Definition~\ref{def:local_mask_identifiability} states when each agent can compute its own product-indexed mask without runtime communication. Thus safety claims for shielded policies mean that the product controller satisfies
$\Phi_{\mathsf{safe}}$, equivalently that its finite-memory projection lies in $\Pi_{\mathsf{safe}}$.

\begin{algorithm}[t]
\DontPrintSemicolon
\caption{Discounted-UCB contract selection}
\label{alg:contract-update}
\KwIn{initial certified contract $\mathcal{C}_0$, certified library $\mathcal{L}_{\mathsf{cert}}$, warm-up $H_0$, dwell time $H$, discount $\eta$, exploration coefficient $\beta$}
deploy $\mathcal{C}_0$ for the first $H_0$ completed episodes\;
initialise discounted statistics $N_{\mathcal{C}},S_{\mathcal{C}} \gets 0$ for every $\mathcal{C}\in\mathcal{L}_{\mathsf{cert}}$\;
set $\mathcal{C} \gets \mathcal{C}_0$ and visited set $V\gets\emptyset$\;
\While{training continues}{
  deploy $\mathcal{C}$ for one dwell block of $H$ completed episodes\;
  observe per-agent episode returns and set $y\gets (Hn)^{-1}\sum_{e=1}^{H}\sum_{i=1}^n R_e^i$\;
  \ForEach{$\mathcal{C}'\in\mathcal{L}_{\mathsf{cert}}$}{
    update $N_{\mathcal{C}'} \gets \eta N_{\mathcal{C}'}$ and $S_{\mathcal{C}'} \gets \eta S_{\mathcal{C}'}$\;
  }
  update $N_{\mathcal{C}} \gets N_{\mathcal{C}}+1$ and $S_{\mathcal{C}} \gets S_{\mathcal{C}}+y$\;
  update $V\gets V\cup\{\mathcal{C}\}$\;
  \eIf{some $\mathcal{C}'\in\mathcal{L}_{\mathsf{cert}}$ has $\mathcal{C}'\notin V$}{
    set $\mathcal{C}$ to the first unvisited certified contract in library order\;
  }{
    set $\mathcal{C} \gets \arg\max_{\mathcal{C}'\in\mathcal{L}_{\mathsf{cert}}}
    \left(S_{\mathcal{C}'}/N_{\mathcal{C}'}+\beta\sqrt{\log(1+\max\{1,\sum_{\tilde{\mathcal{C}}}N_{\tilde{\mathcal{C}}}\})/N_{\mathcal{C}'}}\right)$\;
  }
}
\end{algorithm}

\paragraph{Algorithms.} The learner suite includes independent policy-gradient and value-based baselines (\algo{IPPO}~\cite{de2020independent}, \algo{IQL}~\cite{tan1993multi}), centralised-critic cooperative baselines (\algo{MAPPO}~\cite{yu2022surprising}), constrained baselines, ordinary shields, and contract-aware variants. \algo{PQN-VDN} is a Parallelised Q-Network learner whose selected per-agent Q-values are summed through a VDN team value~\cite{gallici2025simplifying,sunehag2017value}; \algo{ICPO} is the independent \algo{CPO} baseline, using cost critics, conjugate-gradient trust-region steps, and backtracking line search~\cite{achiam2017constrained}.

\paragraph{Synchronisation for contract selection.} During learning, contract selection requires infrequent synchronisation among agents: they agree on the active certified contract and may exchange scalar performance summaries (to aggregate team reward where rewards are not homogeneous), used to test alternatives. This communication occurs only at episode or dwell-block boundaries and is infrequent relative to environment interaction. In particular, dwell block boundaries can be made extremely large, although this slows adaptation toward the best-performing certified contract. It is not part of the runtime safety argument, which is carried by the deployed local certificates; at execution time no such communication is required once local mask identifiability holds.

\paragraph{Training overhead.} The timing tables in Appendix~\ref{app:extended_experiments} show no material overhead from contract selection relative to factorised shielded training, excluding time required to generate the initial library before training begins. Selection is evaluated only at dwell-block boundaries over a finite pre-certified library, so its cost is small compared with environment interaction and policy updates.

\paragraph{Decentralised common reward payoff.} Although our evaluation uses decentralised-training-decentralised-execution algorithms, cooperative MARL optimises team reward. In our experimental configuration, individual rewards are generally aligned with team success, so independent optimisation can improve team return indirectly, even though rewards are neither a homogeneous common payoff nor broadcast to all agents. However, the bandit selector uses team return. All agents share the same task-level preference ordering, so the sum of individual returns is the natural cooperative objective rather than a centralised source of additional safety information.

\section{Hyperparameters} \label{app:hyperparameters}

This appendix records the hyperparameters used for experimental replication. All runs used the same environment configuration, policy optimiser settings, shield-construction parameters, and
contract-selection parameters reported here.
Tables~\ref{tab:hyperparameters_run_budget}, \ref{tab:hyperparameters_ppo},
\ref{tab:hyperparameters_value_based}, and
\ref{tab:hyperparameters_contracts} collect the concrete settings.

\paragraph*{Run budget and benchmark discounts.}
Each experiment setting is run three times with seeds $0,1,2$ and a budget of $10^6$ environment steps per run. All runs use one parallel environment, rollouts of length $128$, four update epochs, and four minibatches per update; for PPO-family methods this gives minibatches of size $32$. The benchmark-level discount factor and episode cap are:

\begin{table}[H]
\centering
\scriptsize
\setlength{\tabcolsep}{3pt}%
\renewcommand{\arraystretch}{1.05}%
\rowcolors{2}{black!3}{white}
\begin{adjustbox}{max width=\linewidth}
\begin{tabular}{@{} l l c c @{}}
\toprule
Benchmark & Safety formula & $\gamma$ & Episode cap \\
\midrule
\env{Pressure Plate} & \texttt{runner\_requires\_return\_routes} & $0.99$ & $150$ \\
\env{RWARE} & \texttt{queue\_yield\_protocol} & $0.995$ & $200$ \\
\env{Car Platoon} & \texttt{maintain\_safe\_gap} & $0.99$ & $60$ \\
\env{Connector} & \texttt{preserve\_reserved\_route\_clearance} & $0.99$ & $12$ \\
\env{Flatland} & \texttt{avoid\_deadlocks} & $0.99$ & $12$ \\
\env{Level-Based Foraging} & \texttt{avoid\_failed\_loads} & $0.99$ & $25$ \\
\bottomrule
\end{tabular}
\end{adjustbox}
\rowcolors{2}{}{}
\caption{Benchmark run configuration used by the experiments.}
\label{tab:hyperparameters_run_budget}
\end{table}

\paragraph*{Policy-gradient settings.}
The PPO-family actor-critic settings in Table~\ref{tab:hyperparameters_ppo}
apply to \algo{IPPO}, \algo{MAPPO}, \algo{Joint-PPO}, their shielded variants,
\algo{Contract-IPPO}, and \algo{IPPO-Lagrangian}. The optimiser is Adam with
$\epsilon=10^{-5}$ and global gradient-norm clipping at $0.5$; the learning
rate is linearly annealed from $2.5\cdot 10^{-4}$ to zero over training.
\algo{ICPO} is specified separately under the constrained settings below.

\begin{table}[H]
\centering
\scriptsize
\setlength{\tabcolsep}{3pt}%
\renewcommand{\arraystretch}{1.05}%
\rowcolors{2}{black!3}{white}
\begin{adjustbox}{max width=\linewidth}
\begin{tabular}{@{} l l @{}}
\toprule
Hyperparameter & Value \\
\midrule
GAE parameter & $\lambda=0.95$ \\
PPO clip parameter & $0.2$ \\
Value clip parameter & $0.2$ \\
Entropy coefficient & $0.01$ \\
Value-loss coefficient & $0.5$ \\
Activation & \texttt{tanh} \\
Actor/value networks & two hidden layers of width $64$ with orthogonal initialisation \\
\algo{MAPPO}/\algo{Joint-PPO} centralised modules & centralised critic for \algo{MAPPO} and centralised actor--critic for \algo{Joint-PPO}, width $64$ \\
\bottomrule
\end{tabular}
\end{adjustbox}
\rowcolors{2}{}{}
\caption{Policy-gradient hyperparameters shared by the PPO-family methods.}
\label{tab:hyperparameters_ppo}
\end{table}

\paragraph*{Value-based settings.}
\algo{IQL}, \algo{Shielded-IQL}, \algo{Contract-IQL}, and
\algo{IQL-Lagrangian} use independent Q-networks with Adam
($\epsilon=10^{-5}$), learning rate $2.5\cdot 10^{-4}$, no learning-rate
annealing, and global gradient-norm clipping at $0.5$. \algo{PQN-VDN} uses the
same optimiser settings and the benchmark discount factor from
Table~\ref{tab:hyperparameters_run_budget}.

\begin{table}[H]
\centering
\scriptsize
\setlength{\tabcolsep}{3pt}%
\renewcommand{\arraystretch}{1.05}%
\rowcolors{2}{black!3}{white}
\begin{adjustbox}{max width=\linewidth}
\begin{tabular}{@{} l l l @{}}
\toprule
Hyperparameter & \algo{IQL} family & \algo{PQN-VDN} \\
\midrule
Network & two hidden layers, width $64$, \texttt{tanh} & two hidden layers, width $256$, \texttt{relu} \\
Normalisation & none & layer normalisation \\
Update epochs & $4$ & $4$ \\
Replay/minibatch size & $256$ & $32$ \\
Learning starts & $5000$ steps & immediate \\
Target update & every $16$ updates, $\tau=1.0$ & online TD target \\
Exploration schedule & $\epsilon:1.0\to0.05$ over $80\%$ of updates & $\epsilon:1.0\to0.1$ over $10\%$ of updates \\
Other settings & double Q-learning enabled; replay size $50{,}000$ & agent id appended; VDN team reward is summed; $\lambda=0.3$; dueling disabled \\
\bottomrule
\end{tabular}
\end{adjustbox}
\rowcolors{2}{}{}
\caption{Value-based learner hyperparameters.}
\label{tab:hyperparameters_value_based}
\end{table}

\paragraph*{Constrained settings.}
The constrained baselines use safety violations as cost signal, cost limit $0$, and cost discount $\gamma_c=1.0$. Both
Lagrangian methods initialise the multiplier at $0$ and update it with step
size $0.05$. \algo{IPPO-Lagrangian} clips the multiplier to $[0,10^6]$, uses
cost GAE parameter $\lambda_c=0.95$, cost-value coefficient $0.5$, and does
not normalise cost advantages. \algo{IQL-Lagrangian} projects the multiplier
to be non-negative and learns independent reward and cost Q-networks with the
IQL-family settings above.
\algo{ICPO} uses independent actor, reward-critic, and cost-critic networks
with hidden widths $(64,32)$, cost GAE parameter $\lambda_c=0.5$, cost-value
coefficient $0.5$, target KL $0.01$, $10$ conjugate-gradient iterations,
damping $0.1$, one Fisher-vector-product sample frequency, and a backtracking
line search with $15$ steps and coefficient $0.8$. Its cost signal is further
augmented by a learned future-violation predictor: a one-hidden-layer
network of width $32$ is trained for $25$ steps per update to predict whether
a raw safety cost occurs within the next $20$ rollout steps, and the predicted
failure probability is added to the raw cost with coefficient $1.0$.

\paragraph*{Shield and contract settings.}
We use discounted-UCB contract selection with discount $0.95$, and contract changes only at episode boundaries. Ordinary shielded baselines use the global safety formula for each agent where viable (i.e. apart from \env{RWARE}, \env{Connector}, and \env{Pressure Plate}, as previously described in Section~\ref{sec:empirical_evaluation}).

\begin{table}[H]
\centering
\scriptsize
\setlength{\tabcolsep}{3pt}%
\renewcommand{\arraystretch}{1.05}%
\rowcolors{2}{black!3}{white}
\begin{adjustbox}{max width=\linewidth}
\begin{tabular}{@{} l r r r r c c c c @{}}
\toprule
Benchmark & \makecell{Max\\cand.} & \makecell{Max\\profiles} & \makecell{Max active\\per agent} & \makecell{Atomic\\prop. cap} & \makecell{Weak\\until} & \makecell{Warmup/\\dwell} & \makecell{UCB\\coef.}  \\
\midrule
\env{Pressure Plate} & $128$ & $1$ & $2$ & $128$ & no & $10/10$ & $1.0$ \\
\env{RWARE} & $1024$ & $12000$ & $2$ & $128$ & no & $10/10$ & $1.0$ \\
\env{Car Platoon} & $1024$ & $12000$ & $2$ & $64$ & no & $10/10$ & $1.0$ \\
\env{Connector} & $1024$ & $4096$ & $2$ & $128$ & no & $0/5$ & $1.0$ \\
\env{Flatland} & $1024$ & $4096$ & $2$ & $96$ & no & $0/5$ & $1.0$ \\
\env{Level-Based Foraging} & $1024$ & $4096$ & $2$ & $128$ & no & $0/5$ & $0.0$ \\
\bottomrule
\end{tabular}
\end{adjustbox}
\rowcolors{2}{}{}
\caption{Contract synthesis and selection hyperparameters. ``Warmup/dwell'' is measured in completed episodes.}
\label{tab:hyperparameters_contracts}
\end{table}

Notice that, in Table~\ref{tab:hyperparameters_contracts}, the UCB coefficient $\beta$ controls the optimism bonus: when $\beta=1$, contracts with fewer recent samples receive a positive exploration bonus. \env{Level Based Foraging} has a contract library that is small and contains profiles that are safe but under-performant in terms of reward. Since the task is sparse reward, setting $\beta=0$ avoids repeated exploration of low-reward profiles.

\paragraph*{Reported curves.}
Reward and safety curves are constructed from completed episode histories.  We report the reporting the mean with $95\%$ confidence interval bands.
Smoothing is applied for visualisation purposes, the plots interpolate each curve to $800$ points and apply a centred moving-average window of width $11$.

\end{document}